\documentclass[11pt]{article}
\usepackage[utf8]{inputenc}
\usepackage[numbers]{natbib}
\usepackage{times}
\usepackage[a4paper, margin=1in]{geometry}

\usepackage{STOC26/stoc-macros}

\usepackage{wrapfig}
\usepackage[percent]{overpic}
\usepackage[format=plain]{caption}

\usepackage{todonotes}
\newcommand{\mnote}[1]{}
\newcommand{\snote}[1]{}
\newcommand{\todonote}[1]{}

\date{}
\title{On the Gradient Complexity of Private Optimization \\ with Private Oracles}

\author{\makebox[1.8in]{Michael Menart  \thanks{Department of Computer Science,
University of Toronto, 
Vector Institute, 
\texttt{michael.menart@utoronto.ca}}}
\makebox[1.8in]{Aleksandar Nikolov \thanks{Department of Computer Science,
University of Toronto, 
\texttt{anikolov@cs.toronto.edu}}
}
}

\newif\ifarxiv
\arxivtrue

\begin{document}

\maketitle

\begin{abstract}%
We study the running time, in terms of first order oracle queries, of differentially private empirical/population risk minimization of Lipschitz convex losses. We first consider the setting where the loss is non-smooth and the optimizer interacts with a private proxy oracle, which sends only private messages about a minibatch of gradients. In this setting, we show that expected running time $\Omega(\min\{\frac{\sqrt{d}}{\alpha^2}, \frac{d}{\log(1/\alpha)}\})$ is necessary to achieve $\alpha$ excess risk on problems of dimension $d$ when $d \geq 1/\alpha^2$. Upper bounds via DP-SGD show these results are tight when $d>\tilde{\Omega}(1/\alpha^4)$. In fact, the lower bound nearly matches the best known upper bound for general private optimizers in this regime. We further show our lower bound can be strengthened to $\Omega(\min\{\frac{d}{\bar{m}\alpha^2}, \frac{d}{\log(1/\alpha)} \})$ for algorithms which use minibatches of size at most $\bar{m} < \sqrt{d}$. We next consider smooth losses, where we relax the private oracle assumption and give lower bounds under only the condition that the optimizer is private. Here, we lower bound the expected number of first order oracle calls by $\tilde{\Omega}\big(\frac{\sqrt{d}}{\alpha} + \min\{\frac{1}{\alpha^2}, n\}\big)$, where $n$ is the size of the dataset. Modifications to existing algorithms show this bound is nearly tight. To our knowledge, ours are the first oracle complexity lower bounds to leverage differential privacy beyond the local privacy model. Compared to non-private lower bounds, our results show that differentially private optimizers pay a dimension dependent runtime penalty. Finally, as a natural extension of our proof technique, we show lower bounds in the non-smooth setting for optimizers interacting with information limited oracles. Specifically, if the proxy oracle transmits at most $\Gamma$-bits of information about the gradients in the minibatch, then $\Omega\big(\min\{\frac{d}{\alpha^2\Gamma}, \frac{d}{\log(1/\alpha)}\}\big)$ oracle calls are needed. This result shows fundamental limitations of gradient quantization techniques in optimization.

\end{abstract}

\section{Introduction}
Many fundamental problems in machine learning and statistics take the form of convex optimization problems. Many such problems can be formulated as empirical risk minimization (ERM), or stochastic convex optimization (SCO)\footnote{Not to be confused with the alternative use of stochastic optimization, where one assumes access to a noisy gradient oracle. 
While related, these settings are distinct. We will refer to this other setting as the stochastic oracle setting.}. For a dataset of $n$ convex losses $\ell_1,\ldots,\ell_n$ and convex constraint set $\cW\subseteq \re^d$ of diameter at most $\rad$, the former asks for an approximate minimizer of the empirical loss: 
$\min_{w\in\cW}\{\cL(w) := \frac{1}{n}\sum_{i=1}^n \ell_i(w)\}$. For an unknown distribution $\cD$, the latter problem is solved by finding an approximate minimizer of population loss $\cL_{\cD}(w) := \E_{\ell\sim\cD}[\ell(w)]$ given independent samples from $\cD$. Such problems have been studied for decades under a variety of regularity conditions on the loss, most commonly $\lip$-Lipschitzness and/or $\beta$-smoothness (i.e. Lipschitz continuous gradients). The runtime efficiency of such algorithms is often measured in the number of first-order oracle calls (i.e. gradient and loss evaluations) made during the optimization procedure. 
Such characterizations date as far back as the work of Nemirovski and Yudin~\citep{nemirovsky85}, which showed that the oracle complexity is $\Theta(\min\bc{d\log(1/\alpha), 1/\alpha^2})$ for minimizing a single (i.e. $n=1$) non-smooth function \footnote{The lower bounds in \cite{nemirovsky85} are loose by log factors for randomized algorithms. This gap has since been closed \cite{BGP17}.}. 

\begin{wrapfigure}{r}{0.5\textwidth}
    \captionsetup{format=plain}
    \begin{overpic}[width=0.7\textwidth]{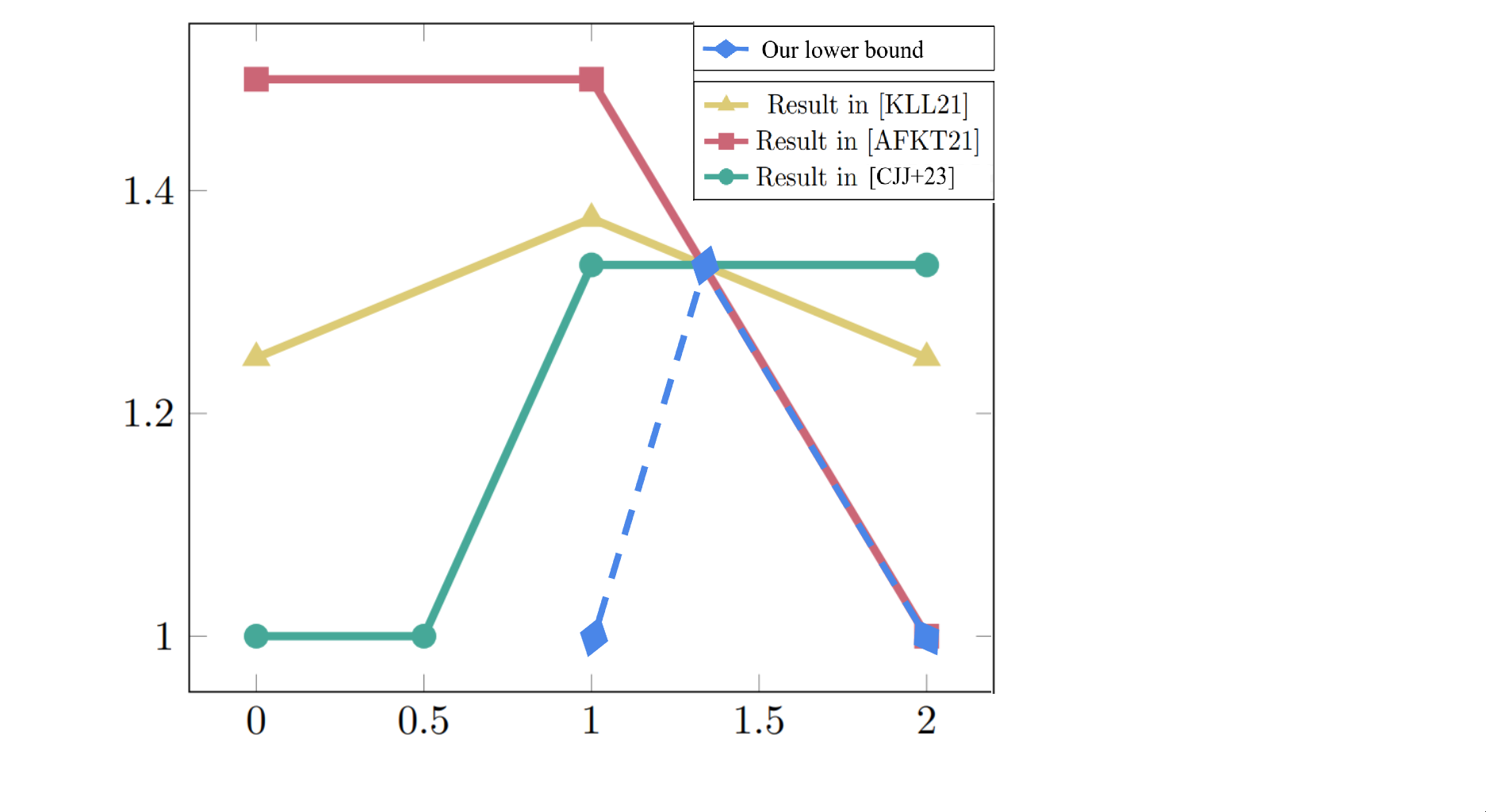}
        \put (27,1) {\small $\kappa: $ dimension $\propto n^\kappa$}
        \put (3,9) {\rotatebox{90}{\small $\beta: $ gradient complexity $\propto n^\beta$} }
    \end{overpic}
    \vspace{-20pt}
    \caption{\small Summary of best known upper bounds and our lower bound for achieving optimal DP-SCO rate $\frac{1}{\sqrt{n}}+\frac{\sqrt{d}}{n}$, ignoring log factors and dependence on privacy parameters. The best upper bounds for private oracle methods is given by the minimum of \cite{AFKT21} and \cite{KLL21}. One work, \cite{resqueing-dp-sco}, improves this rate in low dimensions with an algorithm that does not satisfy oracle privacy.}
    \label{fig:bounds}
\end{wrapfigure}

The power of this framework combined with modern privacy concerns has resulted in a rich literature studying differentially private (DP) analogs of these problems \cite{CMS,BST14,BFTT19,FKT20,KLL21, AFKT21,choquette-choo25a}. For over a decade, it has been known that the best achievable asymptotic rate of the excess empirical risk for ERM under $(\epsilon,\delta)$-DP (DP-ERM) is $\alpha^*_{\epsilon,\delta} := \rad\lip\sqrt{d\log(1/\delta)}/(n\epsilon)$, which was first achieved in $O(n^2)$ running time using DP-SGD,  \citep{BST14}. 
The optimal rate for SCO under $(\epsilon,\delta)$-DP (DP-SCO), which is $\smash{\Theta(\alpha_{\epsilon,\delta}^*+\frac{\rad\lip}{\sqrt{n}}})$, was established subsequently, albeit with even higher runtime overhead \cite{BFTT19}.
Runtimes were eventually improved to $\smash{O\big(\min\{\frac{n^2\epsilon^{3/2}}{\sqrt{d}}, n^{1.5}\big\} \big)}$ in the non-smooth case by \cite{AFKT21} and, in the smooth case, to $O(n)$ by \cite{FKT20}; we omit now factors of $\rad,\lip$ and $\log(1/\delta)$ for simplicity. These remain the fastest known rates for worst case dimension. 
In the non-smooth setting, upper bounds further improve when $\smash{d \leq 
(n\epsilon)^{4/3} \iff d\leq (\alpha_{\epsilon,\delta}^*)^{-4}}$, where \cite{KLL21} achieved time $\smash{O\big(\min\{\frac{n^{3/2}\epsilon}{d^{1/8}},n^{5/4}d^{1/8}\sqrt{\epsilon}\} + \frac{n^2\epsilon^2}{d}\big)}$. Up to $\epsilon$ dependence, \cite{resqueing-dp-sco} improved this rate for DP-SCO to $\smash{\tilde{O}\big(\frac{n^2\epsilon^2}{d}+n^{4/3}\epsilon^{1/3}\big)}$ in the range $d\in [n\epsilon^{2/3}, (n\epsilon)^{4/3}]$,
and rate $\tilde{O}\big(n +  \frac{(nd)^{2/3}}{\epsilon}\big)$ when 
$d<n\epsilon^2$. %
What we are left with is a complex patchwork of runtimes, with little understanding of what is or is not optimal.

Despite a large body of work on improving runtime upper bounds for DP optimization, and the importance of characterizing DP runtime repeatedly cited as an important open problem \citep{bassily2020stability, KLL21,resqueing-dp-sco}, 
oracle complexity lower bounds that leverage DP are virtually non-existent.
Our results complement past work, which has proven fundamental \textit{utility} costs to ensuring privacy, by finally showing that a class of DP optimizers incurs significant running time cost as well. To our knowledge, ours are the first oracle complexity lower bounds to leverage differential privacy beyond the local privacy model.%

\paragraph{Proxy oracles.}
Our main result is a lower bound on the oracle complexity of non-smooth ERM/SCO with access to a private ``proxy'' oracle. %
In our proxy oracle model, at round $t$, an optimizer sends $M_t$ gradient queries ($M_t$ may be adaptively chosen) to the proxy oracle. The proxy oracle then computes a minibatch of $M_t$ gradients using the true gradient oracle, and sends a response to the optimizer. By way of example, the commonly studied stochastic oracle can be considered a type of proxy oracle which responds with noisy estimates of the true gradients. Ideologically, however, stochastic oracles have generally been studied from an adversarial perspective (i.e. worst case noise) whereas we are interested in \textit{cooperative} proxy oracles, i.e. best case proxies satisfying some constraint.
Private proxy oracles are those whose messages are a $\rho$-zCDP (zero concentrated differential privacy) processing of the true oracle responses. 
We remark that even when providing approximate DP guarantees of the overall process, DP optimizers generally use zCDP oracles due to favorable ``moment's accounting'' composition guarantees \cite{abadi2016deep,truncatedCDP}. 
A common case for private oracle methods is for the proxy oracle response to be an estimate of the gradient, often perturbed by Gaussian noise, as in the canonical DP-SGD algorithm.  However, our framework allows arbitrary messages.
To make the most direct comparison to traditional oracle lower bounds, we still measure the overall oracle complexity of the algorithm in the number of calls to the \textit{true} gradient oracle (i.e. the sum of the minibatch sizes). 

Our oracle model is motivated by several factors.
Most obviously, it provides insight into design limitations of private optimizers more generally. Notably, when $d\geq (\alpha_{\epsilon,\delta}^*)^{-4}$, the fastest known method for DP-SCO is a private oracle method \cite{AFKT21}.
Further, despite algorithmic advances in DP optimization, some of which achieve linear (in $n$) time algorithms the smooth or low dimension regimes \cite{FKT20,resqueing-dp-sco}, private oracle methods remain the dominant method in practice, for both convex and non-convex settings \citep{yu21-grad-perturbation-underrated, dp-ml-survey,Cummings2024AdvancingDP}. While often not formalized so explicitly, there is considerable and ongoing effort into studying methods that fall in this class, both in theory and practice \citep{abadi2016deep, KLL21, kairouz2021-practical-dp-deep,choquettechoo2023matrixfactorization, koloskova23, ABGGMU23, menart24, choquette-choo25a}.
One useful aspect of private oracle methods is that they can provide a robust way to preserve privacy even when underlying assumptions about the loss, such as convexity or smoothness, fail to hold. 
Further, when $\rho$ is not too large, strong overall approximate DP guarantees can be obtained via privacy amplification via subsampling or the moments accountant \cite{abadi2016deep, balle2018privacy-amp-via-subsamp}. 
Techniques aside, a myriad of important scenarios naturally lend themselves to private oracle settings. One example is when the optimization procedure is being performed by an untrusted server communicating with nodes holding data, potentially from multiple individuals, as occurs in federated learning \citep{lowy2023privatefederatedlearning}. A common scenario is for the server to make gradient queries to the nodes, and thus ensuring privacy in this setting means the nodes send only privacy preserving messages about the gradients back to server.
Another example is when intermediate models obtained during optimization need to themselves be used or released. Private oracle methods provide a versatile way to guarantee privacy of the entire collection of models generated during training. 
As such, the study of private oracle methods not only serves to provide insight into private optimization more generally, but is a meaningful algorithm class in its own right. 
Further, by studying private oracle methods, we are, as our lower bound shows, able to characterize the effect of \textit{batch sizes} on private training dynamics. Our results formally show the negative impact of small batch sizes on private learning, a phenomenon which has been explored in previous work, but without a minimax characterization \citep{mckenna25a, subsampling-is-not-magic}. This contrasts sharply with the non-private setting, where lower bounds show that $\omega(1)$ batch sizes degrade runtime, at least for $d =\tilde{\Omega}(1/\alpha^4)$ \citep{bubeck19_highlyparallel}.

The private oracle model should be informally viewed as a stronger assumption than assuming a private optimizer, where only the process of producing the final solution must satisfy privacy. Strictly speaking this is not true, since a private oracle does not imply the aggregate procedure is private. For example, the optimizer could become non-private with enough calls. Further, we only assume the private oracle mechanism is DP with respect to the minibatch of \textit{gradients}, which is a weaker assumption than privacy with respect to the minibatch of losses, as multiple gradients in the minibatch could come from a single loss. Regardless, the most interesting comparison comes from methods which design a private optimizer via a private oracle with $\rho \leq 1$, as is often the case in the literature.

\subsection{Results}
\paragraph{Non-smooth loss, private oracle.} In the non-smooth setting when $d>1/\alpha^2$, we show that any optimizer interacting with a private oracle has expected running time $\Omega\big(\min\{\frac{1}{\alpha^2}(\sqrt{\frac{d}{\rho}} + \frac{d}{\bar{m}\rho}), \frac{d}{\log(1/\alpha)}\}\big)$, where $\rho$ is the zCDP privacy parameter of the proxy oracle and $\bar{m}$ is some (possibly infinite) upper bound on the batch sizes. We further show this lower bound is tight for private oracle methods when $d \geq 1/\alpha^4$ by providing a more general analysis of the DP-SGD algorithm. Even in the regime $d\in[1/\alpha^2,1/\alpha^4]$, where our lower bound is $\tilde{\Omega}(d)$, it is still stronger than the non-private lower bound of $\Omega(1/\alpha^2)$.
Of particular note is when $\rho=1$, $\alpha=\alpha^*_{\epsilon,\delta}$, and $\bar{m}=\infty$, where the lower bound is $\Omega(\frac{n^2\epsilon^2}{\sqrt{d}\log(1/\delta)})$, which nearly matches the best known upper bound for general $(\epsilon,\delta)$-DP optimizers in the regime $d\geq (\alpha^*_{\epsilon,\delta})^{-4}$.

The $\tilde{O}(d)$ limit on our lower bound can be removed for algorithms that make at most $d$ \textit{unique} queries to the private oracle. As such, one particularly relevant consequence of our result pertains to the ubiquitous DP-SGD algorithm. Given a minibatch estimate of the gradient, $g_t$, at parameters $w_t$, DP-SGD updates parameters via $w_{t+1}= \Pi_{\cW}[w_t - \eta (g_t + \cN(0,\mathbb{I}_d\sigma^2))]$.
\begin{corollary}(Informal corollary of Theorem \ref{thm:main-lb})
Let $\cA$ be an $\alpha$-accurate (for non-smooth losses) implementation of DP-SGD with batch size $m$. Then its running time is
$\Omega\bigro{\min\bigc{\frac{\sqrt{d} + d/m}{\alpha^2} , \frac{dm}{\log(1/\alpha)}}}$.
\end{corollary}

To our knowledge, outside of the local privacy model, ours are the first results which formally show that a class of private optimizers suffer worse runtime compared to their non-private counterparts. Even for the DP-SGD algorithm specifically, which is the backbone of DP optimization in practice and the subject of intense study, we are unaware of prior work formally showing runtime costs due to privacy.

\paragraph{Non-smooth loss, information limited oracle.} Our proof technique uses information theoretic tools, and thus naturally extends to proxy oracles which transmit at most $\Gamma$ bits of information to the optimizer. 
For such oracles, we show that $\Omega\big(\min\{\frac{d}{\alpha^2\Gamma}, \frac{d}{\log(1/\alpha)}\} \big)$ expected running time is necessary. Compared to classic lower bound constructions in non-smooth optimization with true oracles, which require the optimizer to discover $1/\alpha^2$ vectors/gradients embedded in the loss function, our lower bound shows that, ultimately, the optimizer must indeed use ``the entirety'' of each of the gradients. %
This result establishes fundamental limits for gradient quantization strategies in machine learning, which have received substantial interest due to the prominence of distributed gradient methods \citep{alistarh17-QSGD, stitch18-sparsified-sgd, mayekar20a-RATQ, fartash-adaptive-gradient-quantization, wang23-communication-compression-survey}.

\paragraph{Smooth loss, private optimizer.} In the smooth case, for $(\epsilon,\delta)$-DP ERM algorithms (with access to a true oracle), we show a lower bound on expected running time of $\Omega\bigro{\frac{\sqrt{d}}{\alpha\sqrt{\log(1/\delta)}} + \min\{\frac{1}{\alpha^2}, n\}}$, and provide upper bounds which show this is nearly tight. 
We note previous work has provided tight upper bounds for $\alpha = \alpha^*_{\epsilon,\delta}$ \cite{FKT20}, but by generalizing these methods, we show the lower bound is tight up to log factors for all $\alpha$. 
Our lower bound does not depend on $\epsilon$, which is necessary as our upper bounds satisfy $(\epsilon,\delta)$-DP whenever $\alpha \geq 6\alpha_{\epsilon,\delta}^*$. In other words, stronger privacy guarantees only impact the best achievable accuracy, and come at no running time cost once $\alpha$ is fixed. 
We also note that recent upper bounds show our lower bound cannot be improved via the private oracle model \cite{choquette-choo25a}, at least for optimal error. %

\paragraph{Between DP-ERM and DP-SCO.} 
\ifarxiv
In Appendix \ref{app:erm-to-sco}, we show a privacy preserving procedure which takes a $\alpha$-accurate DP-ERM algorithm, $\cA$, and using $\operatorname{polylog}(n)$ calls to $\cA$, solves DP-SCO with accuracy $\tilde{O}(\alpha + \frac{1}{\sqrt{n}})$, thus incurring only $\operatorname{polylog}(n)$ overhead in runtime, privacy, and excess risk. 
Since $\Omega(\frac{1}{\sqrt{n}})$ error is necessary in SCO, this implies that for any achievable $\alpha$, DP-SCO is no harder than DP-ERM, up to log factors. This result can in fact be obtained from existing results with minimal additional effort.
The reverse reduction was shown in \citep{BFTT19}, and we give a slightly tighter version in Appendix \ref{app:sco-to-erm}. 
For this reason, we focus on DP-ERM in this work.
As an aside, our reduction also means that when $d\leq n\epsilon^2$, accuracy optimal DP-SCO is easier than accuracy optimal DP-ERM simply because of the weaker $1/\sqrt{n}$ accuracy requirement. %
\fi

\subsection{Techniques}

Techniques for lower bounding oracle complexity generally  live in one of two disjoint classes. The first class is the ``vector discovery'' framework, were the loss function is randomly instantiated via a collection of $K>0$ random vectors from $\re^d$. One then tries to argue that, in order to minimize the loss function, the optimizer must observe each vector by having it returned from the oracle. 
As an example, a staple of such techniques is the so-called Nemirovski function, defined as 
$\cN(w)=\max_{j\in[K]}\bc{\ip{w}{X_j} - j\gamma}$, 
for random ``problem vectors'' $X_1,...,X_K\in\re^d$ and offset $\gamma \geq 0$. 
The second class of lower bounds use information theoretic methods. Here, the loss is randomly sampled from a distribution with high entropy. 
One then argues that loss minimization requires obtaining high mutual information with the loss (i.e. identifying the loss instance).

Our technique for non-smooth losses brings together elements from both these approaches.
We use a Nemirovski-like loss function of the form,
$\cL(w) = \max\big\{\max_{k\in[K]}\bc{|\ip{w}{X_{k}}-\alpha|}, \|\Pi_V w\|\big\},$
where $X_1,...,X_K$ are problem vectors in $\re^d$ and $\Pi_V$ is the orthogonal projection onto some subspace $V$, which is orthogonal to $\Span(X_1,\ldots,X_K)$. When $n>1$, our construction simply copies this function $n$ times. In our analysis we show that the optimizer must discover each problem vector. But in contrast to previous analysis in the vector discovery framework, ``discovery'' requires obtaining high mutual information with each $X_k$. This is in part possible  due to the presence of the regularization term, $\|\Pi_V w\|$, which will become more apparent in our proof. A crucial step in our proof is showing that the optimizer is unable to adequately learn $V$ unless it makes $\Omega(d)$ oracle queries.

On the other hand, we also diverge from existing information theoretic techniques in how we track mutual information. Previous techniques essentially upper bound $I(W; \cL)$, where $W$ is the candidate solution. Clearly, we cannot hope to show more than $I(W; \cL) \geq d\log(d/\alpha)$ is necessary, since $w^*=\min_{w\in\cW}\bc{\cL(w)}$ can be communicated in $\tilde{O}(d)$ bits. 
Our solution to this bottleneck is to instead track the sum $\sum_{k=1}^K I(X_k; W | X_{\neq k})$. Under the correct distribution, estimating $X_k$ is still a difficult/high-dimensional problem, even under knowledge of $X_{\neq k}$, allowing us to show that this sum must be $\Omega(dK)$. Crucial to this analysis is bounding the information leaked about $X_k$ when the oracle \textit{does not} return $X_k$, as each oracle response depends on all of $\cL$. 
Previous information theoretic oracle lower bounds often circumvent this issue by considering linear/simple losses, where the query point is largely irrelevant to the gradient. In our case, because we are tracking the conditional information $I(X_k;W|X_{\neq k})$, when $\cO(w)$ does not return $X_k$, we can show the oracle response is just the answer to a $\log(K)$ bit question about $X_k$, and thus leaks little information. This information is small enough that we can then argue the only efficient way to obtain information about $X_k$ is by making oracle queries that return $X_k$. 

For proxy oracles with bounded information capacity, even oracle responses which return $X_k$ may not reveal sufficient information to estimate it. We can then use the above ideas to lower bound the runtime by lower bounding the number of times the optimizer must get the oracle to return $X_k$. In the private case, controlling the amount of information leaked is more technical. While it is true that $\rho$-zCDP mechanisms run on a dataset of (random) size $M$ leak at most $\rho \E[M^2]$ bits, such a bound is too weak to achieve our lower bounds, and a more careful accounting of the information gain must be used.%

\subsection{Related work}
To our knowledge, all existing work which leverages privacy for oracle complexity lower bounds considers the \textit{local} model of privacy, of which the most relevant is \cite{acharya-info-constrained-opt}.  
In the language of our framework, they study the case where the private oracle always has batch size $1$, satisfies $\epsilon$-pure differential privacy, and itself only has access to a stochastic oracle, rather than a true oracle. In this setting, they show $\Omega(\frac{d}{\alpha^2\epsilon^2})$ queries are necessary. In the same setting but with an $\Gamma$-information limited oracle they show a lower bound of $\Omega(\frac{d}{\alpha^2\Gamma})$. 
Their assumption that the private/information limited oracle only has access to a stochastic oracle is significant, and without it their lower bound would lose its polynomial dependence on $\alpha$, as the loss they consider is linear, and so only $\tilde{O}(d)$ bits need to be transmitted before the optimizer can obtain the solution. %

Outside of privacy, there is a substantial literature on oracle complexity bounds. %
For finite sums (i.e. $n > 1$), and sufficiently large dimension, \cite{WS16} proved a complexity lower bound in the non-smooth case of $\tilde{\Omega}\bigro{\min\{\frac{1}{\alpha^2}, n + \frac{\sqrt{n}}{\alpha} \}}$ and, in the smooth case, $\tilde{\Omega}(\min\{\frac{1}{\alpha^2},n + \sqrt{\frac{n}{\alpha}}\})$, which are nearly tight \cite{SVRG}. Their loss construction is very distinct from ours. 
Closer constructions to ours can be found in \cite{bubeck19_highlyparallel} and \cite{MSSV22}, which study the oracle complexity of highly parallelized and memory limited optimization respectively.
Their loss constructions resembles ours in the sense that they use a Nermivoski-like function combined with a ``regularizer'', but in both cases the form and analysis of their regularizer differs substantially from ours.
Examples of works using information theoretic techniques include \citep{Agarwal:2012,BGP17, gopi2022private}.
However, the source or hardness in all these methods come from the difficulty of estimating (at most) $\tilde{O}(d)$ bits, making them fundamentally distinct from our method. We note also \cite{BGP17} considered algorithms with randomized running time, as we do. 

The technique for our lower bound in the smooth case is more related to work on lower bounds for DP mean estimation \citep{BUV14, DSSUV15}. A reduction between (smooth) optimization and mean estimation was shown in \cite{BST14}.

Finally, outside of \citep{acharya-info-constrained-opt}, other works have studied the complexity implications of communication limitations in optimization 
\cite{mayekar20-gradient-compression,mayekar20a-RATQ, huang22, salgia2025characterizingaccuracycommunicationprivacytradeoffdistributed}, but these works still assume the proxy oracle only has access to a stochastic oracle and quantizes only a single gradient. The works \citep{AS15, woodworth21a, scaman19} studied distributed optimization lower bound under different communication constraints, which are not directly comparable to information capacity limitations.

\section{Preliminaries}
\paragraph{Notation.}
We let $\cB(r)$ denotes the $d$-dimensional Euclidean ball
with radius $r$ centered at zero.
For a set of vectors $S$, $\Pi_S$ denotes the orthogonal projection onto $\Span(S)$ and $\Pi_S^\perp$ denotes the projection onto the orthogonal complement; $\Pi_{S,S'}$ is the projection onto $\Span(S)\cup \Span(S')$. When $\cW$ is a compact set, we use $\Pi_\cW$ as the projection onto $\cW$ itself. We let $\rho_X$ denote the law of $X$ and use $\rho(x)$ when disambiguation is obvious by context.
For a collection of random variables, $A_1,...,A_T$, we denote $A=[A_1,...,A_T]$, $A_{\leq t} = [A_1,\ldots, A_t]$ and similarly for $A_{< t}$ and $A_{\neq t}$. 
We define $\alpha_{\epsilon,\delta}^* := \frac{\rad\lip\sqrt{d\log(1/\delta)}}{n\epsilon}$.

\paragraph{Information theory.}

For a discrete random variables $X$ and $Y$, the entropy and mutual information is defined as $H(X)=\sum_x x \log(1/\rho(x))$ and $I(X;Y)=\sum_x \sum_y \rho(x,y)\log\big(\frac{\rho(x,y)}{\rho(x)\rho(y)}\big)$ respectively. %
We take log to be the natural logarithm, such that entropy/information is measure in nats.
For arbitrary random variables, the more general definition can be used;
$I(X;Y) = \sup_{\cP,\cQ}\bc{I([X]_\cP; [Y]_{\cQ})}$,
where the supremum is over all finite partitions of the support, and $[X]_{\cP}$ denotes the quantization of $X$ via the partition of its support, $\cP$  \citep{cover-thomas-elements}.
We define the information capacity of a function as follows.
\begin{definition} [Information Capacity \cite{cover-thomas-elements}]
The information capacity, $\Gamma$, of a randomized function, $f:\cX\mapsto \cY$, is   $\Gamma=\max\limits_{\rho_X}\bc{I(f(X);X)}$, where the maximum is over all distributions supported on $\cX$.
\end{definition}
The most basic example of a function with information capacity at most $\Gamma$ is one whose range is $\{0,1\}^\Gamma$. 
We will also frequently use the fact that the $\alpha$-packing number of $\cB(r)$, i.e. the size of the largest set of vectors $\cV\subseteq \cB(r)$ such that $\forall v,v'\in\cV: \|v-v'\|\geq \alpha$, lies in $[(r/\alpha)^d, (3r/\alpha)^d ]$ \citep{vershynin_hdp}.

\paragraph{Differential privacy.}
An algorithm $\cA$ is \textit{$(\epsilon,\delta)$-differentially private} if for all datasets $S$ and $S'$ differing in one data point and all events $\cE$ in the range of the $\cA$, we have, $\mathbb{P}\bs{\cA(S)\in \cE} \leq     e^\epsilon \mathbb{P}\bs{\cA(S')\in \cE}  +\delta$ \citep{dwork2006calibrating}.
An algorithm $\cA$ is \textit{$\rho$-zero concentrated differentailly private (zCDP)} if for all datasets $S$ and $S'$ differing in one data point and all $\alpha \in (1,\infty)$, it holds that $D_{\alpha}(\cA(S)||\cA(S')) \leq \rho\alpha$, where 
$D_{\alpha}(X||Y) = \frac{1}{\alpha-1}\int \rho_X(x)^\alpha \rho_Y(x)^{1-\alpha} dx$ denotes the $\alpha$-R\'enyi divergence \citep{Bun-zCDP}.

\paragraph{First order optimization.} We consider the problem of minimizing finite sum losses. For a set of $n>0$ losses $\cL=\{\ell_1,...,\ell_n\}$, with some abuse of notation we let  $\cL(w) = \frac{1}{n}\sum_{\ell\in\cL} \ell(w)$. We consider the case where each $\ell_i$ is $\lip$-Lipschitz and possibly also smooth, over a compact convex set $\cW\subset \re^d$ of diameter at most $\rad$. Denoting $w^* = \argmin_{w\in\cW}\bc{\cL(w)}$, we define the suboptimality/excess empirical risk of $w$ as $\cL(w)-\cL(w^*)$. 
In the DP-SCO problem, we assume $\bc{\ell_1,...,\ell_n}\sim\cD^n$ for some distribution $\cD$, and call 
$\E_{\ell\sim\cD}[\ell(w)] - \argmin_{w'\in\cW}\{\E_{\ell\sim\cD}[\ell(w')]\}$ 
the excess population risk. 
First order oracles are a common way to model the interaction between the optimizer and loss function.
\begin{definition}[First Order Oracle]
For losses $\cL=\{\ell_1,...,\ell_n\}$, a first order oracle $\cO=\cO_\cL$
is a function satisfying $\cO(r,i)\in \bc{(\ell_i(r),g): g\in \nabla \ell_i(r)}$; here $\nabla$ denotes the subgradient.
\end{definition}
In our work, we also consider optimizers of the form of Algorithm \ref{alg:private-oracle-optimizer} interacting with a proxy oracle.
\begin{definition}[Proxy Oracle]
Given an oracle $\cO$, a (first order) proxy oracle $\tilde{\cO}$ is an algorithm of the form given by Algorithm \ref{alg:proxy-oracle}. For some range $\cY$ and any side information $\bot$, it is uniquely defined by the set of possibly randomized response functions $\tilde{\cO}_{\bot}: (\re\times \re^d)^* \mapsto \cY$.
\end{definition}
We emphasize that the response of the proxy oracle need not be an estimate of the gradient, or even a vector in $\re^d$. As examples, the oracle could return an estimate of gradient variation (see e.g. \cite{ABGGMU23}), a sketch of the entire gradient minibatch, or even updated model parameters.
When every $\tilde{\cO}_\bot$ has information capacity at most $\Gamma$, we call the oracle $\Gamma$-information limited. When every $\tilde{\cO}_\bot$ is a $\rho$-zCDP mechanism (with respect to its dataset of gradients), we call $\tilde{\cO}$ a $\rho$-\textit{private oracle}. 

By way of example, methods which ensure the optimizer itself is DP by either 1) sampling disjoint minibatches and applying parallel composition or 2) sampling random minibatches and applying advanced composition both fall neatly into the private oracle model. Observe however, that the private oracle model places no restrictions on data reuse or how minibatches are sampled. In this way, it is possible to construct private oracle methods which do not satisfy central DP. It is perhaps surprising that our lower bound still holds for such methods, and arises merely from leveraging the fact that information about the loss passes through a differentially private channel.

\todonote{Expand on data domain condition for private oracle mechanism?}

In the protocol defined by Algorithms \ref{alg:private-oracle-optimizer} and \ref{alg:proxy-oracle} we refer to a \textit{batch} as the set of non-adaptive queries made at some iteration $t\in[T]$. The non-adaptivity assumption is necessary both for this framework to be meaningful and for our lower bounds to hold; otherwise, one could push the entirety of any optimization algorithm into one call of the proxy oracle. In this case, $\cA$ only needs to be a differentially private aggregator of some dataset of $T$ gradients, which is much weaker than even assuming $\cA$ is a private optimizer. Consequently, the Phased SGD method of \cite{FKT20} would yield a faster algorithm; see Appendix \ref{app:smooth-ub}.
\begin{algorithm}[h]
\caption{\textit{Interaction protocol for optimizer, $\cA$, and proxy oracle, $\tilde{\cO}$.}}
\label{alg:private-oracle-optimizer}
\begin{algorithmic}[1]
\REQUIRE Lipschitz parameter $\lip$, Constraint set $\cW$ of diameter at most $\rad$

\STATE Set $t=1$ %

\WHILE{$\cA$ chooses to continue}

\STATE For $M_t\geq 0$, choose $M_t$ queries, $(R_{t,1},I_{t,1}),...,(R_{t,M_t},I_{t,M_t}) \in \re^d\times [n]$ to send to $\tilde{\cO}$

\STATE Receive $Y_t$ from $\tilde{\cO}$

\STATE $t = t+1$
\ENDWHILE

\STATE $T= t-1$, and let $\bar{T}$ be the number of unique vectors in $\bc{R_{t,l}}_{t\in[T],l\in M_t}$ \hfill\textit{(for analysis only)}

\textbf{Output: } $\cA$ releases solution: $\out \in \cW$

\end{algorithmic}
\end{algorithm}

\begin{algorithm}[h]
\caption{\textit{Proxy oracle, $\tilde{\cO}$.}}
\label{alg:proxy-oracle}
\begin{algorithmic}[1]
\REQUIRE Batch size $m>0$, Queries $\bc{(R_{t,l},I_{t,l})}_{l=1}^{m}$%
, Iteration $t>0$, Side information $\bot$

\STATE Compute first order information $G_t=\bc{G_{t,1},...,G_{t,m}}$ where $G_{t,l} = \cO(R_{t,l},I_{t,l})$

\textbf{Output: } $Y_t = \tilde{\cO}_\bot(G_{t,1},...,G_{t,m
})$
\end{algorithmic}
\end{algorithm}

\paragraph{Runtime characterization.}
For any $\rad,\lip,\beta \in \re^+\cup\bc{\infty}$, let $\cF_{\lip,\beta}$ denote the set of all $\lip$-Lipschitz $\beta$-smooth loss functions over $\re^d$ and $\cK_\rad$ denote the collection of all convex sets inside $\cB(\rad)$ \footnote{We consider constrained optimization, but note that analogous result for the unconstrained case are generally obtainable via standard techniques. See \cite[Appendix D.3]{MSSV22} for example.}.
Let
$\instanceruntime(\cA,\tilde{\cO}, \cO_\cL, \cW,\alpha)$ 
denote the expected running time of $\cA$ (measured in the number of evaluations of the true oracle $\cO_{\cL}$) needed to achieve expected suboptimality $\alpha$ on $\cL$, when run with (proxy) oracle $\tilde{\cO}$ on constraint set $\cW$. We then define,
\begin{align*}
\worstruntime(\cA,\tilde{\cO},\alpha,\rad,\lip,\beta) = \sup_{\cW\in\cK_\rad}\sup\limits_{\cL \in \cF^n_{\lip,\beta}}\sup_{\cO_\cL}\bc{\instanceruntime(\cA,\tilde{\cO}, \cO_\cL,\cW,\alpha)},
\end{align*}
where the supremum over $\cO_\cL$ is taken over valid true oracles for $\cL$; note the only flexibility is in how the oracle resolves the subgradient. 
We will consider different classes of algorithms throughout, and so it is useful for notation to omit for now any  quantifiers over $\cA,\tilde{\cO}$ that would specify minimax complexity. 
We say an algorithm is $\alpha$-accurate for $(\cF,\cK)$ if for any $\cL\in\cF$ and $\cW\in\cK$ it yields a solution with expected suboptimality at most $\alpha$.

\section{Non-smooth optimization with private oracles} \label{sec:nonsmooth}
Our main result is a lower bound on the oracle complexity of optimization via private oracles in the large scale regime (i.e. $d \geq 1/\alpha^2$). Compared to optimization with access to a true oracle, whose complexity is $\Theta(1/\alpha^2)$ in this regime, our lower bound shows that optimization via a best-case private oracle incurs a dimension dependent runtime penalty. 
 \begin{theorem} \label{thm:main-lb}
Let $C_2$ be a universal constant. Let $\cA$ be any optimizer satisfying the form of \mbox{Algorithm \ref{alg:private-oracle-optimizer}} and $\ex{}{\bar{T}} \leq \frac{d}{640\log(\rad\lip/\alpha)}$. Let $\tilde{\cO}$ be any proxy oracle such that each $\tilde{\cO}_{\bot}$ is $\rho$-zCDP. 
If
$d \geq \frac{C_2\rad^2\lip^2}{\alpha^2}$ then, 
\ifarxiv
\begin{align*}
    \worstruntime(\cA,\tilde{\cO},\alpha,\rad,\lip,\infty) = \Omega\br{\frac{\lip^2\rad^2\sqrt{d}}{\alpha^2\sqrt{\rho}}}.
\end{align*}
\fi
    If for some $\bar{m}>0$, $\|M\|_\infty \leq \bar{m}$ w.p. $1$, then additionally, 
\ifarxiv
\begin{align*} 
    \worstruntime(\cA,\tilde{\cO},\alpha,\rad,\lip,\infty) = \Omega\br{\frac{\lip^2\rad^2 d}{\alpha^2 \bar{m}\rho}}.
\end{align*}
\fi
 \end{theorem}
 
Before starting the proof, we provide some discussion.
The unique query assumption expands the applicability of the lower bound when considering algorithms such as DP-SGD, which query the oracle many times at a single point each iteration. Observe as a simple corollary however, that we can drop the assumed bound on $\ex{}{\bar{T}}$ and have the lower bound  
$\Omega\big(\min\{\frac{\sqrt{d}}{\alpha^2\sqrt{\rho}}, \frac{d}{\log(1/\alpha)} \}\big)$, and similarly in the case where $\bar{m}$ is bounded \footnote{We will omit factors of $\rad$ and $\lip$ in our discussions for simplicity. They can be obtained by replacing $\alpha$ with $\alpha/(\rad\lip)$}. Also, the same lower bound holds for DP-SCO via a reduction; see Appendix \ref{app:sco-to-erm}.
While it is not clear whether the $\Omega(d/\log(1/\alpha))$ term in this bound is tight, upper bounds from \cite{KLL21} show that for some regime of $\alpha$ and $\rho$, the bound must be weaker than $\frac{\sqrt{d}}{\alpha^2\sqrt{\rho}}$ when $d$ is roughly less than $1/\alpha^4$. Specifically, \cite[Theorem 4.11]{KLL21} provides an algorithm which, for any $\epsilon,\delta\in[0,0.5]$, is $\alpha_{\epsilon,\delta}^*$-accurate, uses a $\rho'$-zCDP oracle with $\rho'=(\frac{\epsilon}{\log(1/\delta)})$, and runs in time $O\bigro{\frac{n^{3/2}\epsilon}{d^{1/8}\log^{1/4}(1/\delta)} + \frac{n^2\epsilon^2}{d\log(1/\delta)}}$.
This is faster than $\frac{\sqrt{d}}{(\alpha_{\epsilon,\delta}^*)^2\rho'}$ %
when $d < \log(1/\delta)/{(\alpha_{\epsilon,\delta}^*)^4}$.

The following upper bound shows the lower bound is tight when $d \geq \frac{\log^2(1/\alpha)}{\alpha^4\rho}$ and $\bar{m} \geq 1/(\alpha^2\rho)$.

\begin{theorem}\label{thm:main-ub}
Let $\alpha,\bar{m},\rho>0$. 
There exists an algorithm of the form given by Algorithm \ref{alg:private-oracle-optimizer} which, using a $\rho$-zCDP proxy oracle, is $O(\alpha)$-accurate for $(\cF_{\lip,\infty}^n,\cK_\rad)$, and runs in at most $O\bigro{\frac{B^2L^2}{\alpha^2}\big(\frac{\sqrt{d}}{\sqrt{\rho}} + \frac{d}{\bar{m}\rho}\big)}$ gradient computations. %
Further, for $\epsilon,\delta \in [0,1]$,
the algorithm is $(\epsilon,\delta)$-DP 
when run with parameters 
$\alpha \geq 26\alpha^*_{\epsilon,\delta}$ 
and $\rho = \frac{1}{\log(1/\delta)}$.
\end{theorem}
The algorithm in question is simply DP-SGD with a careful tuning of the hyperparameters; see Appendix \ref{app:NMSGD-Upper-Bound} for a description and proof.
Further, using our reduction in Appendix \ref{app:erm-to-sco}, this result implies essentially the same upper bound for DP-SCO for any $\alpha \geq 1/\sqrt{n}$.
Note the running time does not depend on the final desired choice of $\epsilon$. Rather, the running time only indirectly depends on $\epsilon$ in that $\epsilon$ affects the minimum achievable error. 
Furthermore, at $\alpha=\Theta(\alpha^*_{\epsilon,\delta})$, the running time is $O\bigro{\frac{n^2\epsilon^2}{\sqrt{d\log(1/\delta)}}}$, which we note improves upon the previous best ERM rates (implicit in \cite{BFTT19,AFKT21}) by a $\sqrt{\epsilon}$ factor.

\begin{remark}
The discrepancy between the privacy notion used for the proxy oracle, zCDP, and the notion used for the final guarantee of DP-SGD, approximate DP, stems from zCDP's inability to be amplified via subsampling and the poor group privacy properties of approximate DP. 
Regardless, most algorithms in the literature use zCDP mechanism even when providing approximate DP guarantees for the overall algorithm. In part, this is because it enables composition guarantees which are tighter than what one would obtain with an approximate DP oracle \cite{abadi2016deep, truncatedCDP}. Regardless, both our upper and lower bounds can be rephrased using the notion of truncated CDP, which is weaker than zCDP and stronger than approximate DP. We provide these details to Appendix \ref{app:tCDP-extension}. Whether the relaxation provided by an approximate DP oracle is meaningful enough to provide stronger upper bounds is a possible direction for further research.
\end{remark}

\subsection{Proof of Theorem \ref{thm:main-lb}} \label{sec:main-proof}
It suffices to consider the case when $\lip=2$, $\rad=1$. This is because, by a standard rescaling reduction, if $\cA$ is $\alpha$-accurate for $(\cF_{\lip,\infty}^n,\cK_\rad)$, 
it can be used to obtain an algorithm which is $\alpha/(\rad\lip)$-accurate for $(\cF_{1,\infty}^n,\cK_1)$. See Fact \ref{fact:rescaling} in Appendix \ref{app:extra-lemmas}.

\paragraph{Hard problem instance.} \label{sec:loss-construction}
To prove our result, we construct a hard distribution over loss functions.
Let $\cW=\cB(1)$. Letting $\ipconst=480$, set $K := \frac{1}{\ipconst^2 \alpha^2}$. We will sample $\bsV$ as a uniformly random $d/2$-dimensional subspace and $X=\bc{X_1,...,X_K}$ as a uniformly random set of orthonormal vectors sampled orthogonal to $\bsV$. Note this is possible since we have assumed $d \geq C_2/\alpha^2$.
We then define
\begin{align*}
    \ell_i(w) = \ell(w) :=\max\big\{\max_{k\in[K]}\bc{f_k(w)}, h(w)\big\}, \quad \forall i \in [n],
\end{align*}
where for all $k\in[K]$,
\begin{align*}
    f_k(w) = |\ip{w}{X_{k}}-\ipconst\alpha|, && \text{ and } && h(w) = 2\|\Pi_{\bsV} w\|.
\end{align*}

As the loss construction of interest is the same for each $i\in[n]$, from here on we will ignore the query indices, $\bc{I_{t,l}}_{t\in[T],l\in[M_t]}$, in the queries made the oracle. At any $w\in\re^d$, we have the true oracle return $(f_k(w),\nabla f_k(w))$ for the smallest valid choice of $k$, and $(h(w),\nabla h(w))$ if no $k$ is valid.

\paragraph{Proof notation. }Before proceeding with the proof, we will need additional notation.
First, we extend the random variables defined in Algorithms \ref{alg:private-oracle-optimizer} and \ref{alg:proxy-oracle} by defining $Y_t = 0$ for $t> T$ and similarly for $R_t$, $G_t$, and $M_t$.
In the following, let $\Y=\bc{\Y_1,Y_2,...}$ and similarly for $M$,$X,G$ and $R$.  
Let $\bc{Q_{t,l}}_{t,l\in\mathbb{Z}^+}$ be the random variables defined as,
{\small
\begin{align*}
    Q_{t,l} = 
    \begin{cases}
    0 & \text{ if } ~t > T \text{ or } l \geq M_t \\
    k & \text{else if } ~ \cO(R_{t,l})=  (f_k(R_{t,l}), \nabla f_k(R_{t,l})) \\
    K+1 & \text{else} \\
    \end{cases}.
\end{align*}}
We denote $Q_t = \bc{Q_{t,l}}_{l\in[M_t]}$.  
For each $k\in[K]$, define, $\cnt_k: \bc{0,\ldots,K+1}^* \mapsto \mathbb{Z}$ as, %
\begin{align*}
    \cnt_k(q) = |\bc{l \in \mathsf{length}(q): q_l = k}| && \text{ and } && \capcnt_k(q) = \min\bc{\cnt_k(q), \sqrt{3\ipconst d/\rho}}.
\end{align*}
In words, $\cnt_k(Q_t)$ is the number of times $\cO$ evaluates via $f_k$ at iteration $t$.
Finally, let $\bar{T}$ be a random variable corresponding to the number of unique vectors in $\bc{R_{t,l}}_{t\in[T],l\in [M_t]}$.

To establish Theorem \ref{thm:main-lb}, we will leverage two main facts. First, the zCDP mechanism $\tilde{\cO}_\bot$ can only leak limited information about the problem vectors contained in the minibatch it act ons. Second, at least $d$ bits of information are needed about each problem vector to successfully solve the optimization problem.
In this regard, it will be helpful to consider the mutual information with respect to a discretization of $X_k$. 
Let $\cC$ be an $\alpha$-packing and $2\alpha$-cover of $\cB(\ipconst \alpha)$.
Note such a set exists, as any maximal $\alpha$-packing is also an $2\alpha$-cover (otherwise, one could find another point to add to the packing, a contradiction). 
We will then characterize the difficulty in estimating 
$\hat{X}_k = \argmin\limits_{c\in\cC}\bc{\|\ipconst\alpha X_k - c\|}$. 

\paragraph{Bounding Information Obtained. }
We will now bound the information obtained about a problem vector in terms of the number of times it is observed by the proxy oracle. 
\begin{lemma} \label{lem:info-ub}
Under the assumptions of Theorem \ref{thm:main-lb}, for any $k\in[K]$ it holds that,
$I(\out;\hat{X}_k|X_{\neq k},\bsV) \leq \E\Big[\sum_{t=1}^\infty \capcnt_k^2(Q_t)\Big]\rho + \ex{}{\bar{T}}\log(K+1),$
where expectation is taken with respect to 
$\cA$, $\tilde{\cO}$, $X$, and $V$.
\end{lemma}
\begin{proof}
In the following condition on $\bsV=\bsv$ and $X_{\neq k}=x_{\neq k}$ for some $\bsv$, $x_{\neq k}$ in their support %
until otherwise stated.
Let $\hat{R}_t$ denote the content sent to $\tilde{\cO}$ at round $t$ by $\cA$. Since 
$I(W;\hat{X}_k) \leq I(Y,\hat{R};\hat{X}_k) \leq I(\Y,\hat{R},Q;\hat{X}_{k})$, 
we start by decomposing the information in $Y$, $\hat{R}$ and $Q$ via the chain rule,
\ifarxiv
{\small
\begin{align} \label{eq:mi-w-xk}
    I(\out;\hat{X}_{k}) %
    &= \sum_{t=1}^\infty I(\Y_t, \hat{R}_t, Q_t;\hat{X}_k | \Y_{<t}, \hat{R}_{<t}, Q_{<t}) \nonumber \\
    &= \sum_{t=1}^\infty I(\hat{R}_t;\hat{X}_k| \Y_{<t}, \hat{R}_{< t}, Q_{<t}) + I(Q_t;\hat{X}_k | \Y_{<t}, \hat{R}_{\leq t}, Q_{<t}) + I(Y_t;\hat{X}_k|\Y_{<t}, \hat{R}_{\leq t}, Q_{\leq t}) \nonumber \\
    &\overset{(i)}{=} \sum_{t=1}^\infty I(Q_t;\hat{X}_k | \Y_{<t}, \hat{R}_{\leq t}, Q_{<t}) + I(Y_t;\hat{X}_k|\Y_{<t}, \hat{R}_{\leq t}, Q_{\leq t}) \nonumber \\ 
    &\leq \sum_{t=1}^\infty \br{\sum_{l=1}^{M_t} H(Q_{t,l};\hat{X}_k|\Y_{<t}, \hat{R}_{\leq t},Q_{< t}, Q_{t,<l})} + I(Y_t;\hat{X}_k|\Y_{<t}, \hat{R}_{\leq t}, Q_{\leq t}). \nonumber \\ 
    &\overset{(ii)}{\leq} \ex{}{\sum_{t=1}^\infty \sum_{l=1}^{M_t}\log(K+1)\cdot \ind{R_{t,l} \notin R_{\leq t,< l}}} + \sum_{t=1}^\infty I(Y_t;\hat{X}_k|\Y_{<t}, \hat{R}_{\leq t}, Q_{\leq t}) \nonumber \\ 
    &\overset{(iii)}{\leq} \ex{}{\bar{T}}\log(K+1) +\sum_{t=1}^\infty I(Y_t;\hat{X}_k|\Y_{<t}, \hat{R}_{\leq t}, Q_{\leq t}). 
\end{align}}
\fi
Step $(i)$ uses the fact that, by data processing, $\hat{R}_t$ contains no information about $\hat{X}_k$ when conditioned on $Y_{< t}$. Step $(ii)$ uses the fact that if a query $R_{t,l}$ is the same as a past query, its (conditional) entropy is zero, and otherwise the entropy is bounded by $\log(K+1)$. The final step $(iii)$ uses the fact that the number of unique queries is $\bar{T}$.

What remains is to bound $\sum_{t=1}^\infty I(Y_t;\hat{X}_k|\Y_{<t}, \hat{R}_{\leq t}, Q_{\leq t})$. To ease notation, define $P_t = (\Y_{<t}, \hat{R}_{\leq t})$. Fix some $t\in[T]$ and note, 
$I(\Y_t;\hat{X}_k|Q_t, P_t) = \E_{q_t\leftarrow Q_t}[I(\Y_t;\hat{X}_k|Q_t=q_t, P_t)].$
Recall $G_t$ is the first order information returned by $\tilde{\cO}$ during round $t$. 
Let us now condition on $Q_t=q_t$ and recall we have already conditioned on $\hat{X}_{\neq k}=\hat{x}_{\neq k}$. Consequently, the only randomness left in $G_t$ is in $X_k$; let $G_t(x_k)$ denote the induced realization of $G_t$ when $X_k=x_k$.
Conditioning on $Q_t=q_t$ and $P_t=p_t$ throughout, and letting $\bot$ be the induced side information at round $t$, we have, 
\begin{align}\label{eq:info-cnt-ub}
    I(\Y_t;\hat{X}_k|Q_t=q_t, P_t=p_t) 
    &\overset{(i)}{\leq} I(\Y_t;X_k|Q_t=q_t, P_t=p_t) \nonumber \\
    &\overset{(ii)}{=}  \ex{x_k\leftarrow X_k}{\KL\Big(\tilde{\cO}_{\bot}(G_{t}(x_k)) ~\big|\big|~ \tilde{\cO}_{\bot}(G_t)\Big) } \nonumber \\
    &\overset{(iii)}{\leq} \ex{x_k, x_k' \leftarrow X_k}{ \KL\Big(\tilde{\cO}_{\bot}(G_{t}(x_k)) ~\big|\big|~ \tilde{\cO}_{\bot}(G_{t}(x_k')) \Big)} \nonumber \\
    &\overset{(iv)}{\leq}  \cnt_k^2(q_t) \rho.
\end{align}

The KL divergence is between the induced conditional distributions given $Q_t=q_t$ and $P_t=p_t$ (as well as $X_{\neq k}=x_{\neq k}$, and $V=v$). %
Above, $(i)$ uses the fact that the mutual information for any quantization of two random variables is upper bounded by the mutual information between the original random variables.
Step $(ii)$ uses the fact that for any random variables $A,B$, $I(A;B) = \E_{B}[\KL(A|B ~||~ A)]$.
Step $(iii)$ uses the fact that the probability distribution of $\cM_t(G_t)$ can be written as the expectation of the conditional distribution given $X_k$ and the convexity of KL divergence. 
Step $(iv)$ uses the definition of zCDP and it's group privacy properties; $\rho$-zCDP implies $s^2\rho$ group zCDP for groups of size $s$. %

We have an additional upper bound on the mutual information via entropy, %
\begin{equation}\label{eq:info-ent-ub}
  I(\Y_t;\hat{X}_k|Q_t=q_t, P_t=p_t) \leq H(\hat{X}_k|Q_t=q_t, P_t) \leq \log(|\cC|) \leq 3\ipconst d.  
\end{equation}
The last inequality comes from upper bounds on packing numbers. 
Recall we have defined $\capcnt_k(q) = \min\{\cnt_k(q), \sqrt{3\ipconst d/\rho}\}$. After consolidating Eqns. \eqref{eq:info-cnt-ub} and \eqref{eq:info-ent-ub} we further have,
\begin{equation*}
    \forall q_t,p_t: ~~ I(\Y_t;\hat{X}_k|Q_t=q_t, P_t=p_t) \leq \capcnt_k^2(q_t)\rho \implies I(\Y_t;\hat{X}_k|Q_t, P_t) %
    \leq \ex{Q}{\capcnt_k^2(Q_t)\rho}.
\end{equation*}
Plugging this into Eqn. \eqref{eq:mi-w-xk}, 
$I(\Y, \hat{R};\hat{X}_k ) \leq \ex{}{\sum_{t=1}^\infty \capcnt_k^2(Q_t)\rho} + \ex{}{\bar{T}}\log(K+1).$
Recall we have conditioned on $\bsV=\bsv$ and $X_{\neq k}=x_{\neq k}$ throughout. Since the above holds for arbitrary instantiations in their support, we have the same upper bound on $I(\out ;\hat{X}_k|X_{\neq k},\bsV)$.    
\end{proof}

\paragraph{Bounding Information Needed. }
We now bound the information needed in the following lemma.
\begin{lemma}\label{lem:info-lb}
Let $d \geq \frac{C_2}{\alpha^2}$.
Under the problem distribution given at the start of Section \ref{sec:main-proof},
if $\cA$ is of the form given by Algorithm \ref{alg:private-oracle-optimizer} with 
~$\ex{}{\bar{T}} \leq \frac{d}{160\log(1/\alpha)}$,
then
~$\min_{k}\{I(\out;\hat{X}_k|X_{\neq k},\bsV)\ \geq d/160$.
\end{lemma}

Due to the length of the proof, we provide a sketch here and defer the full proof to Appendix \ref{app:info-lb}. 
Fix some $k\in [K]$. 
Our information lower bound starts by leveraging a variant of Fano's method, from which we obtain that any estimator $\hat{W}$ satisfies,
\begin{align} \label{eq:dist-lb-main}
\PP\bs{\|\hat{W} - \hat{X}_k\| \geq 40\alpha}
&\geq  \frac{1}{2} - 8\cdot \frac{I(\hat{W};\hat{X}_k |X_{\neq k},\bsV)) + 1}{d}.
\end{align}
With this in hand, we would like to construct an accurate estimator from $\out$ using a bounded amount of additional information.

In this regard, first observe that for every possible instantiation of $X$ and $\bsV$, there exists a minimizer, $w^*=\ipconst\alpha\sum_{k=1}^K X_k$, with $0$ loss. Thus the accuracy condition of the optimization problem alone guarantees that 
$\E[|\ipnos{\out}{X_k} - \ipconst\alpha|] \leq \alpha$ and $\smash{\exnos{}{\|\Pi_V \out\|}} \leq \alpha$. The problematic piece is the component of $\out$ in the ``unpenalized'' subspace, which is $\smash{\Pi_{X,\bsV}^\perp \out}$. This component may be large even if the loss is small. 
Further, while it would be easy enough to project out the components of $W$ in $\Span(X_{\neq k})$ and $\bsV$ because we are considering the conditional mutual information, projecting out the component in the orthogonal complement of $\Span(X)\cup \bsV$ would add too much information, as it localizes $X_k$ to a $K$-dimensional subspace.
A key step will be to show that unless $\cA$ makes enough (i.e. $\Omega(d)$) oracle queries to learn $\bsV$, it cannot leverage the unpenalized subspace.

With this in mind, we now sketch how to construct the modified estimator, $\hat{W}$. Let $Z = \{Z_{t,l}\}_{t\in[T],l\in[M_t]}$ be defined as $Z_t = \nabla h(R_{t,l})$ and let $O=\bc{O_{t,l}}_{t\in[T],l\in[M_t]}$ be such that $O_{t,l}$ is the unit vector orthogonal to $Z_{t,l}$ in the plane spanned by $Z_{t,l}$ and $R_{t,l}$ (regardless of whether or not the oracle response at query $R_{t,l}$ is the gradient of $h$).
The important point for this sketch is that each $Z_{t,l} \in V$ and each $O_{t,l}$ is in the orthogonal complement of $V$. 
We then take $\hat{W}$ roughly equal to $\Pi^{\perp}_{X,\cS}\Pi^{\perp}_{Z} \out$, where $\cS$ is the orthogonal complement of $X_k$ inside $\Span(O)$. The actual construction $\hat{W}$ uses a modified version of $\cS$ to ensure no more than $\frac{d}{160}$ bits of information about $\hat{X}_k$ are added.
Importantly, we can show that so long as the $\cA$ does not take advantage of the unpenalized subspace, $\hat{W}$ accurately estimates $\hat{X}_k$ and thus Eqn. \eqref{eq:dist-lb-main} yields an information lower bound. We finish with the the following lemma, which indeed shows $\cA$ cannot leverage the orthogonal complement of $\Span(X)\cup \bsV$ (the unpenalized subspace).

\begin{lemma} \label{lem:regularization-component-main}
Let $\cS$ be the orthogonal complement of $X_k$ inside $\Span(O)$.
Then under the conditions of Lemma \ref{lem:info-lb},
$\prnos{}{\|\Pi^{\perp}_{X,\cS}\Pi^{\perp}_{Z} \out\| \geq 2\|\Pi_{\bsV} \Pi^{\perp}_{X,\cS}\Pi^{\perp}_{Z} \out\| + 2\alpha} \leq \frac{1}{10}.$
\end{lemma}
\ifarxiv
The key proof ideas are the following. First, when conditioning on $X=x$, we see that the chain $V \rightarrow (O,Z, R) \rightarrow W$ is Markovian. Thus it suffices to reason about the conditional distribution of $\bsV$ given $O$, $Z$ and $R$. Put another way, we can characterize what $\cA$ would learn about $V$ even if it received $\nabla h$ at every query. We show that at best $\cA$ learns $\bsV$ is a subspace that contains $\Span(Z)$ and is orthogonal to $\Span(O)$. Thus, provided the number of queries is not too large, we can show there is a ``leftover'' subspace of $\bsV$, $\Span(V)\setminus \Span(Z)$, which is of dimension $\Omega(d)$ with a conditional distribution that is uniform over an $\Omega(d)$ dimensional space. Now the properties of random projection ensure that the component of $\hat{W}$ in the unpenalized subspace cannot be much smaller than its component in the leftover space. As a result, we can establish that $\exnos{}{\|\hat{W}-\hat{X}_k\|}=O(\alpha)$ and obtain the desired information lower bound from Eqn. \eqref{eq:dist-lb-main}.
\fi
\begin{proof}[Proof of Lemma \ref{lem:regularization-component-main}] 
Throughout we will use the random variables $Z$ and $O$ defined above and those defined by Algorithm \ref{alg:private-oracle-optimizer}. Let $\cV$ be the set of all possible $d/2$ dimensional subspaces of $\re^d$.
Conditioned on a set of gradient oracle queries and responses, let $\cV_{good}=\cV_{good}(R,Z,X) \subset \cV$ denote the set of subspaces which result in the same set of oracle responses from the true oracle $\cO$, and $\cV_{bad} = \cV \setminus \cV_{good}$. 

\paragraph{Conditional distribution of $\bsV$.}
We will first find the conditional density of $v$ after conditioning on all randomness generated during the optimization process. We note the density is with respect to the rotation invariant measures on  $d/2$-dimensional subspaces. Since $O$ and $G$ are completely determined given $X=x,R=r$ and $Z=z$, it suffices to find the conditional density given $X,R$ and $Z$. %
We have,
\begin{align} \label{eq:v-bayes}
    \rho(\bsv|x,r,z,w) = \rho(\bsv|x,r,z) = \frac{\rho(r,z|\bsv,x)\rho(\bsv|x)}{\rho(r,z|x)}.
\end{align}
The first equality uses that
$\bsV \rightarrow (X,R,Z) \rightarrow \out$ is a Markov chain. 
Further, 
{\small
\ifarxiv
\begin{align*}
    \rho(r_{\leq T},z_{\leq T} | \bsv,x) 
    &= \rho(z_T | r_{\leq T}, z_{<T}, \bsv, x)\rho(r_T | r_{<T}z_{<T},\bsv, x) \rho(r_{<T}, z_{<T} | \bsv, x) \\
    &~~\vdots \\
    &= \biggro{\prod\limits_{t=1}^T \rho(z_t | r_{\leq t}, z_{< t}, \bsv, x)} \biggro{\prod\limits_{t=1}^T \rho(r_t | r_{<t},z_{<t},\bsv,x) } \\
    &= \biggro{\prod\limits_{t=1}^T \rho(z_t | r_{t}, \bsv)} \biggro{\prod\limits_{t=1}^T \rho(r_t | r_{<t},z_{<t},x) }.
\end{align*}
\fi
}
 Now plugging into Eqn. \eqref{eq:v-bayes} and using $\rho(z_t | r_{t}, \bsv)=\frac{\rho(z_t, r_{t} | \bsv)}{\rho(r_t | \bsv)}$ we have,
\begin{align*}
    \rho(\bsv|x,r,z,w) = \biggro{\rho(\bsv|x)\prod\limits_{t=1}^T \frac{\rho(z_t, r_{t} | \bsv)}{\rho(r_t | \bsv)}} \biggro{\frac{1}{\rho(r,z|x)}\prod\limits_{t=1}^T \rho(r_t | r_{<t},z_{<t},x) } .
\end{align*}
Observe the second factor on the RHS is independent of $\bsv$ and is thus constant.
For the first factor, since the true oracle $\cO$ is deterministic, $z_t$ is completely determined by $r_{t}$ and $\bsv$. Specifically, for any $\bsv \in \cV_{bad}$, $\bigro{\rho(\bsv|x)\prod_{t=1}^T \frac{\rho(z_t, r_{t} | \bsv)}{\rho(r_t | \bsv)}} = 0$ because $\rho(z_t,r_t|\bsv)=0$ for some $t\in[T]$. 
Alternatively, if $\bsv\in\cV_{good}$, then $\frac{\rho(z_t, r_{t} | \bsv)}{\rho(r_t | \bsv)} = 1, \forall t\in[T]$, and $\bigro{\rho(\bsv|x)\prod_{t=1}^T \frac{\rho(z_t, r_{t} | \bsv)}{\rho(r_t | \bsv)}}$ is constant for all $\bsv \in \cV_{good}$ since $\rho(\bsv|x)$ is a uniform density on subspaces orthogonal to $x$. This establishes that $\rho(\bsv|x,r,z)$ is the uniform distribution over $\cV_{good}$. 

\paragraph{Determining $\cV_{good}$.} 
Our aim is now to prove the following fact: $\cV_{good}$ is exactly the set of $d/2$ dimensional linear subspaces which contain $\Span(z)$ and are orthogonal to $\Span(o)$ and $\Span(x)$.
Let $v\in \mathsf{Supp}(V)$ and consider some individual query $\tilde{r}\in\re^d$ and denote 
$\tilde{z}=\Pi_v \tilde{r}$ and let $\tilde{o}\in\Span(\tilde{r},\tilde{z})$ be the unit vector orthogonal to $\tilde{z}$ with $\ip{\tilde{o}}{\tilde{r}}>0$ or the zero vector if $\tilde{r}=\tilde{z}$.
To prove the fact, it suffices to show that $\tilde{z}=\Pi_v \tilde{r}$ if and only if 
$v$ is a subspace which is orthogonal to $\tilde{o}$ and contains $\tilde{z}$.

The claim is straightforward if $\tilde{r}=\tilde{z}$, and so we focus on the case where $\tilde{r}\neq \tilde{z}$.
We first show that any other subspace, $v'$, which is orthogonal to $\tilde{o}$ and contains $\tilde{z}$ still satisfies $\Pi_{v'} \tilde{r}=\tilde{z}$.
Towards this end, we have the following: 
\begin{align*}
    \min_{u: \ip{u}{\tilde{o}}=0}\bc{\|u-\tilde{r}\|} \overset{(i)}{=} \|\tilde{z} -\tilde{r}\| \overset{(ii)}{\geq} \min_{u\in \Span(v')}\bc{\|u-\tilde{r}\|} \overset{(iii)}{\geq} \min_{u: \ip{u}{\tilde{o}}=0}\bc{\|u-\tilde{r}\|}.
\end{align*}
Equality $(i)$ is because the projection of $\tilde{r}$ to the space orthogonal $\tilde{o}$ is obtained by removing the component along $\tilde{o}$, which results in some vector in the $1$-dimensional space $\Span(\tilde{z})$. By the definition of projection, $\tilde{z}$ is closest point to $\tilde{r}$ in this one dimensional space.
Inequality $(ii)$ follows from the assumption that $\tilde{z}\in \Span(v')$, and $(iii)$ follows from the fact that $\Span(v') \subseteq \{u\in\re^d: \ip{u}{\tilde{o}}=0\}$.
Since the LHS and RHS above are equal, $\|\tilde{z}-\tilde{r}\|=\min_{u\in \Span(v')}\bc{\|u-\tilde{r}\|}$ and thus $\tilde{z} = \Pi_{v'} \tilde{r}$ by the uniqueness of orthogonal projection onto a span.

Now we finish by proving the reverse implication, that if $v'$ does not contain $\tilde{z}$ or is not orthogonal to $\tilde{o}$, then $\Pi_{v'} \tilde{r} \neq \tilde{z}$.
If $\tilde{z} \notin \Span(v')$, clearly $\Pi_{v'}\tilde{r} \neq \tilde{z}$. If $\tilde{z} \in \Span(v')$, but $\Span(v')$ is not orthogonal to $\tilde{o}$, consider some $u \neq \tilde{z}$ as any vector in $\Span(v')$ such that $\ip{u}{\tilde{o}} > 0$. We can assume positive inner product since $\Span(v')$ contains both $u$ and $-u$. Now by definition $v'$ contains the plane spanned by $u$ and $\tilde{z}$.
The properties of orthogonal projection ensure $\ip{\tilde{r}-\Pi_{v'}\tilde{r}}{u-\Pi_{v'}\tilde{r}} \leq 0$ 
Assuming by contradiction that $\Pi_{v'} \tilde{r}=\tilde{z}$, and noting $\tilde{r}=\tilde{z}+a\tilde{o}$ for some $a\geq0$ (recall we assume $\ip{\tilde{o}}{\tilde{r}}>0$), we have 
$\ip{\tilde{z}+a\tilde{o} - \tilde{z}}{u-\tilde{z}} = \ip{a\tilde{o}}{u-\tilde{z}} = a\ip{\tilde{o}}{u} \leq 0$. But since $a\geq 0$, this contradicts the assumption that $\ip{\tilde{o}}{u} > 0$, and thus $\Pi_{v'}\tilde{r} \neq \tilde{z}$. 

\paragraph{Component of output in $\Span(\bsV)$.} 
We have now established $\mathcal{V}_{good}$ is the cartesian product of $\{\Span(Z)\}$ and some set of $d'$ dimensional linear subspaces, for some $d' \geq d/2-\bar{T}$, and that the posterior distribution of $\bsV$ given $R=r$, $Z=z$, $X=x$, and $W=w$ is uniform over $\cV_{good}(r,z,x)$. 
Let $\bar{t}=\bar{t}(r)$ be the value of $\bar{T}$ induced by $R=r$.
Define $U=\Pi_{X,\cS}^\perp \Pi_{Z}^\perp W$ and let $E$ denote the event $\frac{1}{2}\|U\| - \|\Pi_{\bsV} U\|  \geq \alpha$.
We have,
{\small
\begin{align}\label{eq:prob-jl-fail-bound}
    \pr{}{E} &= \int\limits_{\substack{r:\bar{t}(r)\leq d/4 \\ w,x,z}}\pr{}{E | r,w,x,z}\rho(r,w,x,z)drdwdxdz + \int\limits_{\substack{r:\bar{t}(r)> d/4 \\ w,x,z}}\pr{}{E | r,w,x,z}\rho(w,x,z|r)\rho(r)drdwdxdz \nonumber \\
    &\leq \max\limits_{\substack{r:\bar{t}(r)\leq d/4 \\ w,x,z}}\bc{\PP[E|r,w,x,z]} + \frac{1}{20}.
\end{align}}
The inequality uses $\ex{}{\bar{T}}\leq \frac{d}{80}$ and Markov's inequality.
Let $r\in \bc{r: \bar{t}(r)\leq d/4}$ and $w,x,z$ be any possible instantiations of $W,X,Z$ given $R=r$. These quantities determine $U$, and so let $u$ be its realization (recall $O$ is determined given $R$ and $Z$). To bound the worst case probability, observe that under conditioning, $d' \geq d/4$. Thus by the previous analysis, $V$ is a uniformly random subspace of dimension at least $d/4$ supported on a linear subspace of dimension at most $d$, and we can apply the Johnson-Lindenstrauss lemma (see, e.g. \cite[Lemma 5.3.2]{vershynin_hdp}). 
Concretely, for some universal constant $C$,
$\pr{V}{\frac{1}{2}\|u\| - \|\Pi_{\bsV} u\|  \geq \alpha ~\Big|~R=r,W=w,X=x,Z=z} \leq \exp\br{-Cd\alpha^2}.$
Thus when $d \geq \frac{\log(20)}{C \alpha^2}$ (which holds by assumption for $C_2$ large enough), we have via Eqn. \eqref{eq:prob-jl-fail-bound},
\begin{align*}
    \pr{}{\|\Pi^{\perp}_{X,\cS}\Pi^{\perp}_{Z} \out\| \geq 2\|\Pi_{\bsV} \Pi^{\perp}_{X,\cS}\Pi^{\perp}_{Z} \out\| + 2\alpha} \leq \exp(-Cd\alpha^2) + \frac{1}{20} &\leq \frac{1}{10}. \qedhere
\end{align*}
\end{proof} 

\paragraph{Completing the proof of Theorem \ref{thm:main-lb}. }
We conclude the proof of Theorem \ref{thm:main-lb} by applying Lemmas \ref{lem:info-ub} and \ref{lem:info-lb}, which yields,
\ifarxiv $$\frac{d}{160} = \min_{k}\bc{I(W;X_k|X_{\neq k},\bsV)} \leq \ex{Q}{\sum_{t=1}^\infty \capcnt_k^2(Q_t)\rho} + \ex{}{\bar{T}}\log(K+1).$$ \fi
When $\ex{}{\bar{T}} \leq \frac{d}{320\log(1/\alpha^2)}$ we obtain 
$\ex{Q}{\sum_{t=1}^\infty \capcnt_k^2(Q_t)\rho} = \Omega(d)$.
We now consider both cases of the theorem statement. For part $1$, using $\capcnt_k(\cdot) \leq \sqrt{3\ipconst d/\rho}$ one can see that,
$\ex{}{\sum_{t=1}^\infty \capcnt_k^2(Q_t)} = \Omega(d/\rho) \implies \ex{}{\sum_{t=1}^\infty \capcnt_k(Q_t)} = \Omega(\sqrt{d/\rho}).$
Now we obtain,
\begin{align*}
    \ex{}{\|M\|_1} = \Big[\sum_{t=1}^\infty \sum_{k=1}^K \cnt_k(Q_t)\Big] 
    \geq \E\Big[\sum_{k=1}^K \sum_{t=1}^\infty \capcnt_k(Q_t)\Big] %
    = \Omega\Big(\frac{\sqrt{d}}{\alpha^2\sqrt{\rho}}\Big).
\end{align*}
The first part of the theorem is obtained since $\max\limits_{v,x}\{\ex{\cA,\tilde{\cO}}{\|M\|_1 | X=x, V=v}\} \geq \ex{\cA,\tilde{\cO},X,V}{\|M\|_1}$.

If the max batch size is bounded, using the bound $\capcnt_k(Q_t) \leq \bar{m}$ for any $k\in[K],t\in \mathbb{Z}^+$, and proceeding similarly to above we obtain
$\max\limits_{v,x}\{\ex{\cA,\tilde{\cO}}{\|M\|_1 | X=x, V=v}\} \geq \frac{d}{\alpha^2\bar{m}\rho}$ as desired.

\section{Smooth optimization with private optimizers} \label{sec:smooth}
We now turn our attention towards the oracle complexity of DP-ERM for smooth functions.
For such functions, we are able to relax the private oracle assumption and only assume that the entire optimization procedure is differentially private. We also show that our lower bound is tight up to log factors.
\begin{theorem} \label{thm:smooth-lb}
Let $\delta\leq \frac{1}{16 n d}$, $\epsilon\leq \log(1/\delta)$, and $d$ be larger than some constant. Assume $\cA$ is $(\epsilon,\delta)$-DP.
Then,  %
$\worstruntime(\cA,\cO,\alpha,\rad,\lip,\alpha/\rad^2) = \Omega\bigro{\frac{\rad\lip\sqrt{d}}{\alpha\sqrt{\log(1/\delta)}} + \min\bigc{\frac{\rad^2\lip^2}{\alpha^2}, n}}.$
\end{theorem}

\todonote{Mention that part of the difficulty comes from the fact that $\cA$ must discovery which losses have non-zero gradients} We give the proof in Appendix \ref{app:smooth-lb}. Like past lower bounds for \textit{excess risk} (e.g. \cite{BST14}), we leverage the difficulty of private mean estimation and the fact that optimizers can solve mean estimation problems. Concretely, $\alpha$-accurate $(\epsilon,\delta)$-DP mean estimation requires the dataset contain $n \geq \sqrt{d}/\alpha\epsilon$ samples. Our lower bound stems from strengthening this bound when the estimator only observes $s$ samples from the dataset. Some care must be taken here, as it is not necessarily true that the estimator must be $(\epsilon,\delta)$-DP with respect to the observed samples; consider for example, privacy amplification via subsampling. Nonetheless, we can still show the observed samples cannot be ``traceable'', allowing us to provide similar guarantees. For this reason however, the lower bound does not scale with $\epsilon$. Our upper bound, presented subsequently, shows this is necessary. 
The $\min\{\frac{1}{\alpha^2}, n\}$ term in the lower bound holds even for non-private algorithms. This argument again leverages mean estimation, although the details are more straightforward. We are not aware of an existing proof of this non-private lower bound for our exact setting, but certainly very similar results have been obtained previously, e.g. for algorithms with deterministic runtime by \cite{WS16}.

The lower bound is mostly matched by a modification of the Phased SGD algorithm of \cite{FKT20}.
\begin{theorem} \label{thm:smooth-ub}
Let $\alpha>0$, $\delta \in [0,1]$, and $\beta\leq\lip\sqrt{d\log(1/\delta)}/\rad$. There exists an algorithm which is $O(\alpha)$-accurate for $(\cF^n_{\lip,\beta},\cK_{\rad})$ and 
uses at most $O\bigro{\max\big\{\frac{\rad\lip\sqrt{d\log(1/\delta)}
}{\alpha}, \frac{\rad^2\lip^2}{\alpha^2}\big\}}$ oracle evaluations. Further, for $\epsilon\in[0,1]$, if $\alpha \geq 6\alpha_{\epsilon,\delta}^*$ it satisfies $(\epsilon,\delta)$-DP.
\end{theorem}

We provide a proof in Appendix \ref{app:smooth-ub}. This nearly matches the lower bound when $\alpha \geq 1/\sqrt{n}$. 
In the low error regime with $\beta$-smooth losses, an alternative algorithm based on solving a series of regularized ERM problems with accelerated ERM solvers can achieve similar results in roughly $O((n + \beta n/\sqrt{d})\log^2(n/\alpha))$ running time. A similar approach for DP-SCO was given in \cite{ZTOH22}. We provide details for this result in \ref{app:linear-smooth-ub}. In aggregate, these upper bounds imply the lower bound is essentially tight. 
Given that our upper bound does not use a private oracle, one question that arises is whether the lower bound, which holds for general DP algorithms, can still be matched by \textit{private oracle} algorithms. At least in the case of $1$-smooth losses and $\alpha=\alpha^*_{\epsilon,\delta}$, the results of \cite{choquette-choo25a} show the answer is yes. For $\omega(1)$-smooth losses it is unclear.

\section{Non-smooth optimization with information limited oracles} \label{sec:info-capped}

Our proof techniques from Section \ref{sec:nonsmooth} can easily be adapted to provide lower bounds for information limited oracles. In this section, we consider an individual loss $\cL$, i.e. $n=1$, such that the objective is to approximate $\argmin_{w\in\cW}\bc{\cL(w)}$. We are interested in algorithms of the form of Algorithm \ref{alg:private-oracle-optimizer} (ignoring the query indices) interacting with a proxy oracle of bounded information capacity. 

\begin{theorem} \label{thm:info-cap-oracle-lb}
Let $\tilde{\cO}$ be a $\Gamma$-information limited proxy oracle and $\cA$ an algorithm of the form given by Algorithm \ref{alg:private-oracle-optimizer} with $\ex{}{\bar{T}} \leq \frac{d}{640\log(\rad\lip/\alpha)}$.
If $d \geq \frac{C_2\rad^2\lip^2}{\alpha^2}$ then
$\worstruntime(\cA,\tilde{\cO},\alpha,\rad,\lip,\infty) = \Omega\bigro{\frac{\rad^2\lip^2 d}{\alpha^2 \Gamma}}$.
\end{theorem}
Once again, we can drop the unique query limit and more simply lower bound the runtime as
$\Omega\big(\min\{ \frac{d}{\alpha^2 \Gamma}, \frac{d}{\log(1/\alpha)} \}\big)$. 
Note also that it is possible to have $\Gamma = \omega(d)$, which can be reasonable when the batch sizes are $\omega(1)$.
As a corollary of our result, consider the case where we fix the batch size in Algorithm \ref{alg:private-oracle-optimizer} to be $1$ and the proxy oracle is instantiated to be $\tilde{\cO}(w) = \nabla \cL(w) + \cN(0,\mathbb{I}_d \frac{\sigma^2}{d})$. A standard fact on Gaussian channels implies that for $1$-Lipschitz losses and $\sigma \geq 1$ this oracle has information capacity $\Gamma \leq d/\sigma^2$. Theorem \ref{thm:info-cap-oracle-lb} thus recovers the oracle complexity lower bound for stochastic oracles, $\Omega\big(\frac{\sigma^2}{\alpha^2}\big)$ \cite{nemirovsky85}, at least for certain parameter regimes. 
This $\frac{\sigma^2}{\alpha^2}$ lower bound is achieved by SGD, which incidentally also means our lower bound is tight when $\Gamma \geq 1/\alpha^2$. That said, it 
is perhaps more interesting to consider whether the bound can be matched via a quantization scheme; \cite{mayekar20-gradient-compression, mayekar20a-RATQ} shows this is the case up to log factors.

Despite recovering stochastic and quantized stochastic oracle complexity lower bounds, we emphasize that our lower bound also meaningfully diverges from such results. Such lower bounds have been obtained via a reduction to mean estimation, where each oracle response is a noisy version of this mean, e.g. \cite{nemirovsky85, Agarwal:2012}), and in the stochastic quantized oracle setting, strong data processing techniques are used to get better bounds \cite{mayekar20-gradient-compression, acharya-info-constrained-opt}. 
Clearly, we cannot hope to obtain Theorem \ref{thm:info-cap-oracle-lb} from such constructions, as transmitting the mean vector to $\alpha$ accuracy requires only $O(d\log(1/\alpha))$ bits of information. Put another way, the difficulty in previous lower bound constructions largely stems from the difficulty of mean estimation. Our bound relaxes the stochastic oracle assumption by leveraging structure unique to solving optimization problems. For similar reasons, it is not a-priori obvious that allowing the proxy oracle use a batch size larger than $1$, and thus transmit messages about multiple gradients, would not help improve oracle complexity. Our lower bound shows this is indeed the case.

\paragraph{Proof of Theorem \ref{thm:info-cap-oracle-lb}.}
The proof leverages the same loss construction and distribution as Theorem \ref{thm:main-lb}. In particular, we will apply Lemma \ref{lem:info-lb} verbatim. Upper bounding the information is in fact simpler than in Lemma \ref{lem:info-ub}.
\begin{lemma} \label{lem:ic-info-ub}
Under the conditions of Theorem \ref{thm:info-cap-oracle-lb}, for any $k\in[K]$ it hold that
\begin{align} \label{eq:ic-oracle-info-ub}
    I(\out;X_k|X_{\neq k},\bsV) \leq \E\Big[\sum_{t=1}^\infty \ind{\cnt_j(Q_t)\geq 1}\Big] \Gamma + \ex{}{\bar{T}}\log(K+1).
\end{align}
\end{lemma}
\begin{proof}
As in the proof of Lemma \ref{lem:info-ub}, we condition on $V=v$ and $X_{\neq k} = x_{\neq k}$ throughout and recall the definitions of $Q$, $\cnt$, and $P_t=(Y_{< t}, \hat{R}_{\leq t})$. Using the same derivation as Eqn. \eqref{eq:mi-w-xk} we obtain,
\begin{align} \label{eq:mi-w-xk-2}
    I(\Y, R;X_{k}) &\leq 
    \ex{}{\bar{T}}\log(K+1) +\sum_{t=1}^T I(Y_t;X_k|Q_{\leq t},P_t). 
\end{align}
Further by the assumption on the oracle, when $Q_t=q_t$,
\begin{align*}
    I(Y_t;X_k|Q_t=q_t, P_t) \leq \Gamma \cdot \ind{\cnt_j(Q_t)\geq 1} \implies I(Y_t;X_k|Q_t, P_t) \leq \Gamma \cdot \ex{}{\ind{\cnt_j(Q_t)\geq 1}}.
\end{align*}
Recalling we have conditioned on $V=v$, we take expectation and plug into Eqn. \eqref{eq:mi-w-xk-2} to obtain the claim.
\end{proof}

\begin{proof}[Theorem \ref{thm:info-cap-oracle-lb}]
Under the assumption that $\ex{}{\bar{T}} \leq \frac{d}{320\log(1/\alpha^2)}$, applying Lemmas \ref{lem:info-lb} and \ref{lem:ic-info-ub} obtains,
\begin{align*}
    \Omega(dK) \leq \sum_{j=1}^K I(Y, R;X_j|X_{\neq j} V) &\leq \ex{}{\sum_{t=1}^\infty \sum_{j=1}^K \ind{\cnt_j(Q_t)\geq 1} }\Gamma + K\ex{}{\bar{T}}\log(K+1) \\
    &\leq \ex{}{\sum_{t=1}^\infty \sum_{j=1}^K M_t }\Gamma + K\ex{}{\bar{T}}\log(K+1) \\
    &= \ex{}{\|M\|_1}\Gamma + K\ex{}{\bar{T}}\log(K+1).
\end{align*}
We can then finish similarly to Theorem \ref{thm:main-lb}.
\end{proof}

\paragraph{Acknowledgements.} 
Michael Menart would like to thank Raef Bassily and Crist\'obal Guzm\'an for the insights gained while working with them on earlier attempts at this problem. This research was supported by an NSERC Discovery Grant (RGPIN-2021-03206), and the Canada Research Chairs program (CRC-2020-00004).

\bibliographystyle{alphaurl}
\bibliography{ref}

\appendix

\section{Supplementary Lemmas} \label{app:extra-lemmas}
The following lemmas and fact will be used several times throughout the appendix.

\begin{fact}\label{fact:rescaling}
Let $\cA$ be an algorithm with expected running time $T$ and which is $\alpha$-accurate for $(\cF_{1,\infty},\cW_1)$ and $\cW_1$. Then one run of $\cA$ can be used in a black box manner to obtain an $\alpha\rad\lip$-accurate algorithm for $(\cF_{\lip,\infty},\cW_\rad)$ with the same oracle complexity.
\end{fact}
The above is a standard fact and comes from running $\cA$ on constraint set $\hat{\cW} = \bc{\frac{w}{\rad}: w\in\cW}$ and loss $\hat{\cL}(w) = \cL(w/\lip)$, and then rescaling the output of $\cA$ by $\rad$.

\begin{lemma}\label{lem:proj-commute}
Let $E,F$ be linear subspaces of $\re^d$ which are orthogonal to each other. Then $\Pi_E^\perp \Pi_F^\perp = \Pi_F^\perp \Pi_E^\perp = \Pi_{F,E}^\perp $.
\end{lemma}
\begin{proof}
Let $d_E = \Dim(E)$ and $d_F = \dim(F)$. Let $u_1,...,u_{d_E}$ be an orthonormal basis for $E$, and $u_{d_E+1},...,u_{d_E+d_F}$ be an orthonormal basis for $F$. Let $u_{d_E+d_F+1},...,u_{d}$ be an orthonormal basis for the remaining space. For some vector $v\in\re^d$, let $\gamma_j = \ip{v}{u_j}$. Now clearly
\begin{align*}
    \Pi_E^\perp v = \Pi_E^\perp \sum_{j=1}^d \gamma_j u_j = \sum_{j=d_{E}+1}^d \gamma_j u_j,
\end{align*}
and similarly for $\Pi_{F}^\perp$. It is now easy to see the projections commute. Because they commute, the product of projections is equal to the projection onto the intersection.
\end{proof}

We will use the following result on privacy amplification via subsampling with replacement, which is a minor modification of  \cite[Lemma 4.14]{bun-dp-thresholds}.
\begin{lemma}\label{lem:amp-wor}
Let $\epsilon,\delta\in[0,1]$ and let $\cM$ be an $(\epsilon,\delta)$-DP algorithm for datasets of size $m > 0$. Then if $\epsilon \leq \min\bc{1,\frac{n}{2m}}$, the algorithm $\tilde{\cM}$, which on input dataset $S$ of size $n$, first samples $m$ points with replacement and then runs $\cM$ on the result is  $(\epsilon',\delta')$-DP with
\begin{align*}
    \epsilon' = 6\epsilon \frac{m}{n}  && \text{and} && \delta' = 4e^{(6\epsilon m/n)}\frac{m}{n}\delta.
\end{align*}
\end{lemma}
In contrast to the original statement, this lemma applies when $m > n$. Obviously, in this regime the result does not amplify privacy, and rather controls the impact of the likely event in which a datapoint gets copied into the sampled dataset many times. Nonetheless, this unified phrasing will be convenient. The proof is nearly identical. Bounding $\epsilon$ follows in exactly the same way, and to bound $\delta$ we leverage our additional assumption that $\epsilon \leq \min\bc{1,\frac{n}{2m}}$. We have copied the proof from \cite{bun-dp-thresholds}, with the necessary modification, below.

\begin{proof}[Proof of Lemma \ref{lem:amp-wor}]
Let $D, D^{\prime}$ be adjacent databases of size $n$, and suppose without loss of generality that they differ on the last row.
Let $T$ be a random variable denoting the multiset of indices sampled by $\tilde{\cM}$ and let $\ell(T)$ be the multiplicity of index $n$ in $T$. Fix a subset $S$ of the range of $\tilde{\cM}.$ For each $k=0,1,...,m$ let
\begin{align*}
    p_{k}&=\PP[\ell(T)=k]=\binom{m}{k}n^{-k}(1-1/n)^{m-k}=\binom{m}{k}(n-1)^{-k}(1-1/n)^{m}, \\
    q_{k}&=\PP[\cM(D|_{T})\in S|\ell(T)=k], \\
    q_{k}^{\prime}&=\PP[\cM(D^{\prime}|_{T})\in S|\ell(T)=k].
\end{align*}

Here, $D|_{T}$ denotes the subsample of $D$ consisting of the indices in $T$, and similarly for $D^{\prime}|_{T}$. Note that $q_{0}=q_{0}^{\prime}$, since $D|_{T}=D^{\prime}|_{T}$ if index $n$ is not sampled. Our goal is to show that
$$\PP[\tilde{\cM}(D)\in S]=\sum_{k=0}^{m}p_{k}q_{k}\le e^{\epsilon'}\sum_{k=0}^{m}p_{k}q_{k}^{\prime}+\delta'=e^{\epsilon'}\PP[\tilde{\cM}(D^{\prime})\in S]+\delta'.$$
To do this, observe that by privacy, $q_{k}\le e^{\epsilon}q_{k-1}+\delta$ so
\begin{align*}
    q_k \leq e^{k\epsilon}q_0 + \frac{e^{k\epsilon}-1}{e^\epsilon-1}\delta
\end{align*}

Hence,
\begin{align*}
\PP[\tilde{\cM}(D)\in S]&=\sum_{k=0}^{m}p_{k}q_{k} \\
&\le\sum_{k=0}^{m}\binom{m}{k}(n-1)^{-k}(1-1/n)^{m}\left(e^{k\epsilon}q_{0}+\frac{e^{k\epsilon}-1}{e^{\epsilon}-1}\delta\right) \\
&=q_{0}(1-1/n)^{m}\sum_{k=0}^{m}\binom{m}{k}\left(\frac{e^{\epsilon}}{n-1}\right)^{k}+\frac{\delta}{e^{\epsilon}-1}(1-1/n)^{m}\sum_{k=0}^{m}\binom{m}{k}\left(\frac{e^{\epsilon}}{n-1}\right)^{k}-\frac{\delta}{e^{\epsilon}-1} \\
&=q_{0}(1-1/n)^{m}\left(1+\frac{e^{\epsilon}}{n-1}\right)^{m}+\frac{\delta}{e^{\epsilon}-1}(1-1/n)^{m}\left(1+\frac{e^{\epsilon}}{n-1}\right)^{m}-\frac{\delta}{e^{\epsilon}-1} \\
&=q_{0}\left(1-\frac{1}{n}+\frac{e^{\epsilon}}{n}\right)^{m}+\frac{\left(1-\frac{1}{n}+\frac{e^{\epsilon}}{n}\right)^{m}-1}{e^{\epsilon}-1}\delta. \quad (1)
\end{align*}
Similarly, we also have that
$$\PP[\tilde{\cM}(D^{\prime})\in S]\ge q_{0}\left(1-\frac{1}{n}+\frac{e^{-\epsilon}}{n}\right)^{m}-\frac{\left(1-\frac{1}{n}+\frac{e^{-\epsilon}}{n}\right)^{m}-1}{e^{-\epsilon}-1}\delta,$$
Combining inequalities 1 and 2 we get that
$$\PP[\tilde{\cM}(D)\in S]\le\left(\frac{1-\frac{1}{n}+\frac{e^{\epsilon}}{n}}{1-\frac{1}{n}+\frac{e^{-\epsilon}}{n}}\right)^{m}\cdot\left(\PP[\tilde{\cM}(D^{\prime})\in S]+\frac{1-\left(1-\frac{1}{n}+\frac{e^{-\epsilon}}{n}\right)^{m}}{1-e^{-\epsilon}}\delta\right)+\frac{\left(1-\frac{1}{n}+\frac{e^{\epsilon}}{n}\right)^{m}-1}{e^{\epsilon}-1}\delta,$$
proving that $\tilde{\cM}$ is $(\epsilon',\delta')$-DP for
\begin{align*}
    \epsilon' \leq m \log\br{\frac{1+\frac{e^\epsilon-1}{n}}{1+\frac{e^{-\epsilon}-1}{n}}} \leq \frac{6\epsilon m}{n}.
\end{align*}
Using $m \geq n/2$ and $\epsilon \leq \min\{1,\frac{n}{2m}\}$, we have,
\begin{align*}
   \delta' &\leq e^{6\epsilon \frac{m}{n}} \frac{1-\exp\br{\frac{2m}{n}(e^{-\epsilon}-1)}}{1-e^{-\epsilon}}\delta + \frac{\exp\br{\frac{m}{n}(e^\epsilon-1)}-1}{e^\epsilon - 1}\delta \\
   &\leq e^{6\epsilon \frac{m}{n}} \frac{1-\exp\br{-\frac{2m}{n}\epsilon}}{1-e^{-\epsilon}}\delta + \frac{\exp\br{\frac{2m}{n}\epsilon}-1}{e^\epsilon-1}\delta \\
   &\leq e^{6\epsilon \frac{m}{n}}\br{\frac{2m}{n}}\delta + (\frac{2m}{n})\delta \\
   &\leq e^{6\epsilon \frac{m}{n}}\br{\frac{4m}{n}}\delta. \qedhere
\end{align*}
\end{proof}

\section{Supplement to Section \ref{sec:nonsmooth}}
\subsection{Proof of Lemma \ref{lem:info-lb} (information lower bound)}\label{app:info-lb}
\begin{lemma}\label{lem:info-lb-restatement} (Restatement of Lemma \ref{lem:info-lb})
Let $d \geq \frac{C_2}{\alpha^2}$ for some constant $C_2$.
Under the problem distribution given at the start of Section \ref{sec:main-proof},
if $\cA$ is of the form given by Algorithm \ref{alg:private-oracle-optimizer} with 
~$\ex{}{\bar{T}} \leq \frac{d}{160\log(1/\alpha)}$,
then (for any proxy oracle)
\begin{align*}
   \min_{k}\bc{I(\out;\hat{X}_k|X_{\neq k},\bsV)} \geq \frac{d}{160}.
\end{align*}      
\end{lemma}
We note that the lemma does not assume any structure of the proxy oracle, and thus holds even for the case when the optimizer interacts with the true oracle.
Before giving the proof, we will require some additional notation which will also be used in later lemmas.
Let $Z=\bc{Z_{t,l}}_{t\in[T],l\in[M_t]}$ be the random variables such that $Z_{t,l} = \nabla h(R_{t,l})$. Further let $O=\bc{O_{t,l}}_{t\in[T],l\in[M_t]}$ be such that $O_{t,l}$ is the unit vector orthogonal $Z_{t,l}$ in the plane spanned by $Z_{t,l}$ and $R_{t,l}$, taken with $\ip{O_{t,l}}{R_{t,l}} \geq 0$ to break ambiguity. If $Z_{t,l}$ and $R_{t,l}$ are colinear, define $O_{t,l}$ to be the zero vector. 

Finally, we remark that Lemma \ref{lem:proj-commute} (Appendix \ref{app:extra-lemmas}) will be used several times throughout the proof, and the results to reach Eqn. \eqref{eq:dist-lb} are given in the subsequent subsection, Appendix \ref{app:lemmas-for-info-lb}.
We will now proceed with the proof. 
\begin{proof}[Proof of Lemma \ref{lem:info-lb-restatement}]
Fix some $k\in [K]$. 
Our information lower bound starts by leveraging a variant of Fano's method, whose details we defer to Lemma \ref{lem:fano-info-lb} in Appendix \ref{app:lemmas-for-info-lb}. Via this lemma, we have that any estimator $\hat{W}$ satisfies,
\begin{align} \label{eq:dist-lb}
\PP\bs{\|\hat{W} - \hat{X}_k\| \geq 40\alpha}
&\geq  \frac{1}{2} - 8\cdot \frac{I(\hat{W};\hat{X}_k |X_{\neq k},\bsV)) + 1}{d}.
\end{align}
The rest of the proof will be devoted to showing we can construct such an estimator from $\out$ and a bounded amount of additional information.

\paragraph{Constructing estimator, $\hat{W}$.} 
Our first step is to transform the output of $\cA$ into a vector which is close (in Euclidean distance) to $\hat{X}_k$. To do this we, roughly speaking, need to remove several components of $\out$: the component contained in $\Span(X_{\neq k})$, the component in $\Span(O)$, and the component in $\Span(Z)$. 
 
Let $\cC_O$ be a minimal $\alpha$-cover of $\Span(O)$ (chosen deterministically given $O$) and define $\tilde{X}_k = \min\limits_{c\in\cC_O}\bc{\|\Pi_{O}X_k - c\|}$.
Now let $\cS$ be the orthogonal complement of $X_k$ inside $\Span(O)$ and $\tilde{\cS}$ be the orthogonal complement of $\tilde{X}_k$ inside $\Span(O)$.
Finally, let $\hat{W} = \Pi^\perp_{X_{\neq k},\tilde{\cS}} \Pi_{Z}^\perp \out$. 
Intuitively, $\Pi_{X_{\neq k},\tilde{\cS}}^\perp$ approximately removes the component of $\out$ in the subspace spanned by $O$ that does not overlap with $X_k$. The projection $\Pi_Z^\perp$ can be interpreted as removing the component of $\out$ in the ``known'' span of $V$.
It is worth noting that $\hat{W}$ is still an accurate solution in some sense, a fact we will now show in more detail.

\paragraph{Showing $\hat{W}$ is accurate. }
Analyzing now the distance term inside the probability on the LHS of Eqn. \eqref{eq:dist-lb} above, 
we have,
\begin{align} \label{eq:dist-ub}
    \|\hat{W} - \hat{X}_k\| \leq \|\hat{W}-\ipconst\alpha X_k\| + 2\alpha
    \leq 2\max\bc{\|\Pi^\perp_{X_k} \hat{W} \|, |\ipnos{\hat{W}}{X_k} - \ipconst\alpha|} + 2\alpha.
\end{align}
The first inequality uses the fact that $\cC$ is a $2\alpha$-cover.
We will now show that both terms in the maximum are bounded with constant probability using the accuracy condition.
Towards this end, observe that for every possible instantiation of $X$ and $\bsV$, there exists a minimizer, $w^*=\ipconst\alpha\sum_{k=1}^K X_k$, with $0$ loss. Recall $K=\frac{1}{\ipconst^2 \alpha^2}$, so indeed $w^* \in\cB(1)$.

We now analyze the inner product term in the RHS of Eqn. \eqref{eq:dist-ub}.
Since the minimizer has $0$ loss, by the accuracy condition on $\cA$, it must be that
$\ex{}{|\ipnos{\out}{X_k} - \ipconst\alpha|} \leq \ex{}{F(W)-F(w^*)} \leq \alpha$, 
and so by Markov's inequality,
$\pr{}{|\ipnos{\out}{X_k} - \ipconst\alpha| \geq 10\alpha} \leq \frac{1}{10}$.
Thus it suffices to show $\ipnos{\hat{W}}{X_k} \approx \ipnos{W}{X_k}$.
We first have,
\begin{align*}
    |\ipnos{W}{X_k} - \ipnos{\hat{W}}{X_k}| = |\ipnos{W}{X_k} - \ipnos{W}{\Pi_{\cS}^\perp X_k}| = |\ipnos{W}{\Pi_\cS X_k}|.   
\end{align*}
The first equality follows from the fact that $\hat{W} = \Pi_{Z}^\perp \Pi^{\perp}_{X_{\neq k}} \Pi^\perp_{\cS} \out$ and $Z$ and $X_{\neq k}$ are orthogonal to $X_k$. The second equality uses $X_k = \Pi_\cS X_k + \Pi_\cS^\perp X_k$. 
Continuing,
\begin{align*}
    \|\Pi_\cS X_k\| \overset{(i)}{=} \|\Pi_\cS \Pi_O \tilde{X}_k + \Pi_\cS  \Pi_O (X_k-\tilde{X}_k)\| \overset{(ii)}{=} \|\Pi_\cS \Pi_O (X_k-\tilde{X}_k)\| \leq \|\Pi_O (X_k-\tilde{X}_k)\| \overset{(iii)}{\leq} \alpha.
\end{align*}
Here, $(i)$ uses the fact that $\cS \subseteq \Span(O)$, $(ii)$ uses the fact that $\Pi_\cS \tilde{X}_k = 0$, 
and $(iii)$ uses the fact that $\tilde{X}_k \in \Span(O)$ and 
$\|\Pi_O X_k - \tilde{X}_k\|\leq \alpha$.
We now obtain 
$|\ipnos{\hat{W}-\out}{X_k}| \leq 2\alpha$ from Cauchy Schwartz and the fact that $\|\hat{W}-W\|\leq 1$. 

This inner product difference with the previously derived fact that 
$\pr{}{|\ipnos{\out}{X_k} - \ipconst\alpha| \geq 10\alpha} \leq \frac{1}{10}$ finally yields,
\begin{align} \label{eq:ip-diff-via-acc}
    \pr{}{|\ipnos{\hat{W}}{X_k} - \ipconst\alpha| \geq 12\alpha} \leq \frac{1}{10}.
\end{align}

We now address the norm term in Eqn. \eqref{eq:dist-ub}.
We have,%
\begin{align*}
    \ex{}{\|\Pi_V\Pi_{X,\cS}^\perp\Pi_{Z}^\perp W\|} &= \ex{}{\|\Pi_V\Pi_{X,\cS}^\perp(\Pi_Z^\perp W - W + W)\|} \\
    &\leq \ex{}{\|\Pi_Z^\perp W - W\| + \|\Pi_V\Pi_{X,\cS}^\perp W\|} \\
    &\leq \ex{}{\|\Pi_Z W\|} + \ex{}{\|\Pi_V W\|} \\
    &\leq \alpha. 
\end{align*}
Above, we have used the fact that $V$ is orthogonal to $\Span(X) \cup \cS$ and Lemma \ref{lem:proj-commute}. The last inequality uses the accuracy condition, since $\cL(w)-\cL(w^*) \geq 2\|\Pi_V w\|$. Continuing,
\begin{align*}
    \ex{}{\|\Pi_V\Pi_{X,\cS}^\perp\Pi_Z^\perp W\|} \leq \alpha 
    &\overset{(i)}{\implies} \pr{}{\|\Pi_V\Pi_{X,\cS}^\perp\Pi_Z^\perp\out\| \geq 5\alpha} \leq 1/5 \\
    &\overset{(ii)}{\implies} \pr{}{\|\Pi_{X,\cS}^\perp\Pi_Z^\perp \out\| \geq 12\alpha} \leq 3/10 \\
    &\overset{(iii)}{\iff} \pr{}{\|\Pi_{X_k}^\perp \Pi_{X_{\neq k},\cS}^\perp\Pi_Z^\perp \out\| \geq 12\alpha} \leq 3/10 \\
\end{align*}
Implication $(i)$ uses Markov's inequality, and $(ii)$ results from Lemma \ref{lem:regularization-component-main}. %
For implication $(iii)$, we apply Lemma \ref{lem:proj-commute} since $X_k$ is orthogonal to $X_{\neq k}$ and $\cS$.

Now observe for any $u\in\cB(1)$, $\|\Pi_{X_{\neq k},\cS}^\perp u - \Pi_{X_{\neq k},\tilde{\cS}}^\perp u\| \leq \alpha$ since $\Span(X_{\neq k})\cup \cS$ and $\Span(X_{\neq k}) \cup \tilde{\cS}$ are close. That is, for any $v\in\cB(1)\cap \cS$, there exists $\xi\in \Span(O)\cap \cB(\alpha)$ s.t. $v+\xi \perp \tilde{X}$, and thus $v+\xi \in \tilde{S}$. 
Thus we obtain,
\begin{align}
     \pr{}{\|\Pi_{X_k}^\perp \Pi_{X_{\neq k},\tilde{\cS}}^\perp\Pi_Z^\perp \hat{W}\| \geq 13\alpha} = \pr{}{\|\Pi_{X_k}^\perp \hat{W}\| \geq 13\alpha} \leq 3/10.
\end{align}
Using this probability bound and Eqn. \eqref{eq:ip-diff-via-acc} above we obtain,
$$\pr{}{\max\bc{\|\Pi^\perp_{X_k} \hat{W} \|, |\ipnos{\hat{W}}{\hat{X}_k} - \ipconst\alpha|} \geq 12\alpha} \leq \frac{1}{10}+\frac{3}{10}\leq \frac{2}{5}.$$
Combing this fact with Eqns. \eqref{eq:dist-lb} and $\eqref{eq:dist-ub}$ we obtain,
\begin{align*}
    \frac{2}{5} \geq \PP\bs{2\max\bc{\|\Pi^\perp_{X} \hat{W} \|, |\ipnos{\hat{W}}{X_k} - \ipconst\alpha|} + 2\alpha \geq 40\alpha} \geq \frac{1}{2} - 8\cdot \frac{I(\hat{W};\hat{X}_k | X_{\neq k},\bsV) + 1}{d}.
\end{align*}
Consequently we have,
\begin{align} \label{eq:reconstruction-info-lb}
    \frac{d}{80} \leq I(\hat{W};\hat{X}_k | X_{\neq k},\bsV).
\end{align}

\paragraph{Showing $\hat{W}$ does not add much information. }
We now show that generating $\hat{W}$ does not add too much information beyond what is contained in $\out$. Let $\hat{R}$ denote the information sent to $\tilde{\cO}$ over the training run. Observe, 
\begin{align} \label{eq:reconstruction-info-add}
 I(\hat{W};\hat{X}_k | X_{\neq k}, \bsV) 
 &\leq I(W, O, Z, \tilde{X}; \hat{X}_k| X_{\neq k}, \bsV) \nonumber \\
  &\leq I(Y, \hat{R}, O, Z, \tilde{X}; \hat{X}_k| X_{\neq k}, \bsV) \nonumber \\
 &= I(\Y, \hat{R}; \hat{X}_k | X_{\neq k}, \bsV) + I(O, Z; \hat{X}_k | X_{\neq k}, \bsV, \Y, \hat{R}) + I(\tilde{X}_k ; \hat{X}_k | X_{\neq k}, \bsV, \Y, \hat{R}, O, Z) \nonumber \\
 &\overset{(i)}{\leq} I(\Y, \hat{R}; \hat{X}_k| X_{\neq k}, \bsV) + H(\tilde{X}_k | O) \nonumber \\
 &\overset{(ii)}{\leq} I(\Y, \hat{R}; \hat{X}_k | X_{\neq k}, \bsV) + \ex{}{\bar{T}}\log(3/\alpha) \nonumber \\
 &\leq I(\Y, \hat{R}; \hat{X}_k | X_{\neq k}, \bsV) + \frac{d}{160} .  
\end{align}
The first inequality uses the fact that $\hat{W}$ is determined by $W$, $O$, $Z$, and $\tilde{X}$.
Step $(i)$ uses that $O$ and $Z$ are deterministic conditioned on $\bsV$ and $\hat{R}$. 
Line $(ii)$ uses $H(\tilde{X}_k|O) \leq \ex{}{\log(|\cC_O|)} \leq \ex{}{d_O}\log(3/\alpha)$.

Finally, combining Eqns. \eqref{eq:reconstruction-info-add} and \eqref{eq:reconstruction-info-lb}, which hold for any choice of $k$, we obtain,
\begin{align*}
    \frac{d}{160} &= \min_{k}\bc{I(W ; \hat{X}_k|X_{\neq k},\bsV)}.  
    \qedhere
\end{align*} 
\end{proof}

\subsubsection{Fano variant used in proof of Lemma \ref{lem:info-lb-restatement}/Lemma \ref{lem:info-lb}} \label{app:lemmas-for-info-lb}
\begin{lemma} \label{lem:fano-info-lb}
For any estimator $\hat{W}$, it holds that,
\begin{align*}
\PP\bs{\|\hat{W} - \hat{X}_k\| \geq 40\alpha} 
&\geq  \frac{1}{2} - 8\cdot \frac{I(\hat{W};\hat{X}_k |X_{\neq k},\bsV)) + 1}{d}.
\end{align*}
\end{lemma}
\begin{proof}
Recall $\hat{X}_k = \argmin\limits_{c\in\cC}\bc{\|\ipconst\alpha X_k - c\|}$ and note $\cC$ has dimension $d_{\cC} = d - d/2 - K + 1 \geq d/4$ (since $d \geq C_2/\alpha^2$).
It can be shown that $\hat{X}_k$ has large entropy.
We defer this fact to Lemma \ref{lem:entropy-after-snapping} given below, and here apply this lemma to obtain
$H(\hat{X}|X_{\neq k} = x_{\neq k}, V=v) \geq (d_{\cC}-1)\log(\ipconst\alpha/2\alpha) = (d_{\cC}-1)\log(\ipconst/2)$.
Define,
{\small
\begin{align*}
    P_{err}=\PP\bs{\|\hat{W} - \hat{X}_k\| \geq 40\alpha ~|~X_{\neq k}=x_{\neq k},\bsV=\bsv } && \text{and} && N_{max}=\max\limits_{c\in\cC}\bc{|\{ c'\in\cC : \|c-c'\|\leq 40\alpha \} |}.
\end{align*}}
Now by a variant of Fano's inequality (see Lemma \ref{lem:distance-fano} below),
\begin{align*}
    P_{err} \geq \frac{H(\hat{X}_k|X_{\neq k}=x_{\neq k},\bsV=\bsv)-\log(N_{max})}{\log(|\cC|/N_{max})} - \frac{I(\hat{W};\hat{X}_k |X_{\neq k}=x_{\neq k},\bsV=\bsv)+1}{\log(|\cC|/N_{max} - 1)}.
\end{align*}
Observe $N_{max}$ is at most the number of $\alpha$-balls that can be packed into a $40\alpha$-ball in $d_{\cC}$-dimensions. Bounds on the packing number then imply $N_{\max} \leq 120^{d_\cC}$ and $|\cC| \in [\ipconst^d, (3\ipconst)^{d_{\cC}}]$. 
Now for $\ipconst = 480$ and $d$ larger than some constant (recall $d_{\cC} \geq d/4$) we obtain,
\begin{align*}
    P_{err} 
    &\geq \frac{(d_{\cC}-1)\log(\ipconst/2)-d_{\cC}\log(120)}{d_{\cC}\log(1440)} - \frac{I(\hat{W};\hat{X}_k |X_{\neq k}=x_{\neq k},\bsV=\bsv)+1}{d_{\cC}\log(\ipconst/120 - 1)} \\
    & \geq \frac{1}{2} - 8\cdot \frac{I(\hat{W};\hat{X}_k |X_{\neq k}=x_{\neq k},\bsV=\bsv)+1}{d}
\end{align*}
Taking the expectation over $X_{\neq k}$ and $\bsV$ then yields,
{\small
\begin{align*} 
\PP\bs{\|\hat{W} - \hat{X}_k\| \geq 40\alpha} &= \ex{X_{\neq k}, \bsV}{\PP\bs{\|\hat{W} - \hat{X}_k\| \geq 40\alpha |
X_{\neq k}=x_{\neq k},\bsV=v) }} \nonumber\\
&\geq  \frac{1}{2} - 8\cdot \frac{I(\hat{W};\hat{X}_k |X_{\neq k},\bsV)) + 1}{d}. \qedhere
\end{align*}}
\end{proof}

\begin{lemma} \label{lem:entropy-after-snapping}
Let $\cC$ be a maximal $\alpha$-packing of the radius $r$ ball in $d_{\cC}$ dimensions. Let $X$ be a uniformly random unit vector and $\hat{X} = \argmin_{c\in\cC}\bc{\|x-c\|}$. Then $H(\hat{X}) \geq (d_{\cC}-1)\log(r/2\alpha)$.
\end{lemma}
\begin{proof}
In the following $\cS(r,c)$ be the surface of the $d_{\cC}$-dimensional ball of radius $r$ centered at $c$. Since $\cC$ is a $2\alpha$ cover and $X$ is uniform over the surface of a $d_{\cC}$-dimensional ball, for any $c\in\cC$, $\PP[\hat{X} = c]$ can be bounded using the surface area, i.e. $(d_\cC-1)$-dimensional volume, of the spherical cap containing points within $2\alpha$ distance of $c$ on a $r$-radius ball. 
Now observe that the surface area of this cap is at most the surface area of a ball of radius $2\alpha$. 
To see this, let $\cX$ be the convex hull of $\bc{x\in\cS(r,0): \|x-c\|\leq 2\alpha}$. Taking the convex hull does not decrease the measure. Further, the mapping $f:\cS(r,c)\mapsto \cX$ given by $f(x)=\Pi_\cX(x)$ is a contraction, and because $\cX$ is convex the image of $\cS(r,c)$ under $f$ is $\cX$. Thus $\mathsf{Vol}(\cS(r,c)) \geq \mathsf{Vol}(\{f(x): x\in\cS(r,c)\}) = \mathsf{Vol}(\cX)$.

Now since $X$ has uniform density of value $\frac{1}{\mathsf{Vol}(\cS(r,0))}$, we have 
$\PP[\hat{X} = c] \leq \frac{\mathsf{Vol}(\cS(2\alpha,0))}{\mathsf{Vol}(\cS(r,0))} = \frac{(2\alpha)^{d_{\cC}-1}\mathsf{Vol}(\cS(1,0))}{r^{d_{\cC}-1}\mathsf{Vol}(\cS(1,0))} = (2\alpha/r)^{d_{\cC}-1}$. 
This establishes that
$H(\hat{X}) = \ex{}{\log(1/\PP[\hat{X} = \hat{x}])} \geq (d_{\cC}-1)\log(r/2\alpha)$. 
\end{proof}

\begin{lemma}\label{lem:distance-fano}
(Restatement of \cite[Proposition 1]{duchi2013distancebasedfano})
Let $X$ and $Y$ be random variables supported on the discrete set $\cC$. Let $\tau \geq 0$ and define $N_{max}=\max_{c\in\cC}\bc{|\{ c'\in\cC : \|c-c'\|\leq \tau \} |}$. Then
\begin{align*}
    \pr{}{\|X-Y\| \geq \tau} \geq \frac{H(X)-\log(N_{max})}{\log\br{\frac{|\cC|}{N_{max}}}} - \frac{I(Y;X)+1}{\log\br{\frac{|\cC|}{N_{max}} - 1}}.
\end{align*}
\end{lemma}
\begin{proof}
Define $N_{min}=\min_{c\in\cC}\bc{|\{ c'\in\cC : \|c-c'\|\leq \tau \} |}$.
The claim is obtained from a simple manipulation of \cite[Proposition 1]{duchi2013distancebasedfano}. Starting from that statement we have the following,
\begin{align*}
    \pr{}{\|X-Y\| \geq \tau} &\geq \frac{H(X|Y)-\log(N_{max})-1}{\log\br{\frac{|\cC|-N_{min}}{N_{max}}}} \\
    &\geq \frac{H(X)-\log(N_{max})}{\log\br{\frac{|\cC|-N_{min}}{N_{max}}}} - \frac{I(Y;X)+1}{\log\br{\frac{|\cC|-N_{min}}{N_{max}}}} \\
    &\geq \frac{H(X)-\log(N_{max})}{\log\br{\frac{|\cC|}{N_{max}}}} - \frac{I(Y;X)+1}{\log\br{\frac{|\cC|}{N_{max}}-1}}. \qedhere
\end{align*}
\end{proof}

\subsection{Proof of Theorem \ref{thm:main-ub} (upper bound via DP-SGD)} \label{app:NMSGD-Upper-Bound}
\begin{algorithm}[H]
\caption{\label{alg:NMSGD} DP-SGD}
\begin{algorithmic}[1]
\REQUIRE Oracle privacy $\rho \leq 1$, Batch size bound $\bar{m}$, Accuracy $\alpha\geq 0$, Oracle $\cO_{\cL}$, 
Constraint Set $\cW$ of width $B$, Lipschitz constant $L$

\STATE Pick any $w_{0}\in\cW$

\STATE If $\alpha \geq \frac{\rad\lip}{3}$, stop and release $w_0$

\STATE Set $m=\min\bc{\sqrt{d/\rho}, \bar{m}}
$, and $\sigma=L\max\bc{\frac{1}{\sqrt{d}},\frac{1}{\bar{m}\sqrt{\rho}}}$ 
\STATE Set $T=\frac{B^2L^2}{\alpha^{2}} \cdot \max\bc{1,\frac{d}{\bar{m}^2\rho}}$, $\eta=\frac{B}{L\sqrt{T}}\cdot \min\bc{1,\frac{\bar{m}\sqrt{\rho}}{\sqrt{d}}}$, 

\FOR{$t=1...T$}
\STATE Sample minibatch $i_1,...,i_m$ uniformly from $[n]$

\STATE Obtain gradients $G_{t}$ from $\{\cO_{\cL}(w_t,i_1),...,\cO_{\cL}(w_t,i_m)\}$

\STATE $w_{t+1}=\Pi_{\cW}\left[w^{t}-\eta\left(\frac{1}{m}\sum_{g\in G_{t}}g+\xi^{t}\right)\right]$, where $\xi_{t}\sim\cN(0,\mathbb{I}_{d}\sigma^{2})$
\ENDFOR
\STATE \textbf{Output:} $\bar{w}=\frac{1}{T}\sum_{t=1}^{T}w^{t}$
\end{algorithmic}
\end{algorithm}

The fact that Algorithm \ref{alg:NMSGD} uses a $\rho$-zCDP oracle comes directly from the guarantees of the Gaussian mechanism \cite{Bun-zCDP}. Theorem \ref{thm:main-ub} follows from the subsequent two lemmas.
\begin{lemma}
Algorithm \ref{alg:NMSGD} is $\alpha'$-accurate for $(\cF_{\lip,\infty}^n,\cK_\rad)$ with $\alpha'=O(\min\bc{\rad\lip,\alpha})$.
Further, the algorithm has oracle complexity $O\bigro{\frac{B^2L^2}{\alpha^2}\big(\frac{\sqrt{d}}{\sqrt{\rho}} + \frac{d}{\bar{m}\rho}\big)}$.
\end{lemma}
\begin{proof}
The result is essentially a corollary of known convergence results for SGD with noise.
For example, by Lemma 3.3 of [BFTT19], Algorithm \ref{alg:NMSGD} obtains excess empirical risk,
\begin{equation}
\mathbb{E}[\loss(\bar{w})-\loss(w^{*})]=O\left(\frac{B^{2}}{\eta T}+\eta L^{2}+\eta\sigma^{2}d\right).
\end{equation}
Plugging in the parameters settings verifies the utility guarantee. The $\rad\lip$ term in the error comes from the trivial bound when $w_0$ is released.
The oracle complexity is $Tm =  O\br{\frac{B^2L^2}{\alpha^2}\big(\frac{\sqrt{d}}{\sqrt{\rho}} + \frac{d}{\bar{m}\rho}\big)}$.
\end{proof}
\begin{lemma} \label{lem:nmsgd-aggregate-privacy}
Let $\delta \in [0,1]$.
\label{thm:nmsgd_privacy} Algorithm \ref{alg:NMSGD} run with $\alpha \geq 26\alpha_{1,\delta}^*$ and 
$\rho = \frac{1}{\log(1/\delta)}$ 
is $(\epsilon,\delta)$-DP with 
$\epsilon = 3\frac{\alpha^*_{1,\delta}}{\alpha}$.
\end{lemma}
\begin{proof}
We will use truncated differential privacy to perform the analysis, and will thus use several results from \cite{truncatedCDP}. By the guarantees of the Gaussian mechanism, the private oracle satisfies $(\rho,\infty)$-tCDP (or equivalently $\rho$-zCDP). 
Provided that 
$\log(n/m) = \log(n\sqrt{\rho/d}) \geq 3\rho(2+\log_2(1/\rho))$, we can apply the privacy amplification via subsampling from \cite[Theorem 11]{truncatedCDP}. 
Note we can assume $\frac{\sqrt{d\log(1/\delta)}}{n}\leq 1/3$ (otherwise the algorithm releases $w_0$). 
Thus for $\rho \leq \frac{1}{\log(1/\delta)}$, 
$\log(n\sqrt{\rho/d})=\log(\frac{n}{\sqrt{d\log(1/\delta)}}) \geq 1$, and so our setting of $\rho$ satisfies the subsampling lemma conditions. 
We thus obtain that each iteration satisfies 
$(\rho',\omega')$-tCDP with $\rho' = 13(m/n)^2\rho = 13d/n^2$ and $\omega' = \frac{1}{4\rho}.$ The composition properties of tCDP imply the overall algorithm satisfies $(\rho'',\omega'')$-tCDP with
$\rho'' = T\rho' = \frac{13\rad^2\lip^2d}{\alpha^2n^2}$ and $\omega''=\omega' = \frac{1}{4\rho}$.
We can now apply the following conversion to $(\epsilon,\delta)$-DP, given by \cite[Lemma 6]{truncatedCDP},
\begin{align*}
    \epsilon = \begin{cases}
        \rho''+2\sqrt{\rho''\log(1/\delta)} & \text{ if } \log(1/\delta) \leq (\omega''-1)^2\rho'' \\
        \rho''\omega'' + \frac{\log(1/\delta)}{\omega''-1} & \text{ if } \log(1/\delta) \geq (\omega''-1)^2\rho''
    \end{cases}
\end{align*}
Observe that under our setting of $\rho$, the condition $\log(1/\delta)\leq (\omega''-1)^2\rho''$ is satisfied whenever $\alpha \geq 26\frac{\rad\lip\sqrt{d\log(1/\delta)}}{n}=26\alpha^*_{1,\delta}$, in which case we obtain privacy $\epsilon = 3\frac{\alpha_{1,\delta}^*}{\alpha}$. 
\end{proof}

\subsection{Extension to tCDP} \label{app:tCDP-extension}
In this section, we show how our lower and upper bound in the non-smooth setting hold under the notion of truncated CDP \cite{truncatedCDP}.
\begin{definition}(Truncated Concentrated Differential Privacy)
Let $\rho > 0$ and $\omega \geq 1$. A randomized algorithm $\cM:\cX^n\mapsto \cY$ satisfies $\omega$-truncated $\rho$-concentrated differential privacy, $(\rho,\omega)$-tCDP, if for all datasets, $S,S'$, differing in at most one element, it holds that for all $\alpha\in(1,\omega)$ that 
$D_{\alpha}(\cM(X) || \cM(S')) \leq \rho\alpha$, where $D_{\alpha}$ is the $\alpha$-R\'enyi divergence. %
\end{definition}

\paragraph{Lower bound.} Consider the case where each $\cO_\bot$ satisfies $(\rho, \omega)$-tCDP for $\omega \geq \min\{\sqrt{d/\rho},\bar{m}\}$. For comparison, note that $\rho$-zCDP is equivalent to $(\rho,\infty)$-tCDP. To extend the lower bound to this notion requires only a slight modification to Lemma \ref{lem:info-ub}, which upper bounds the information gained during the optimization procedure. Specifically, we recall Eqn. \eqref{eq:info-cnt-ub}, which showed that under $\rho$-zCDP (using the same notation),
\begin{align}
    I(\Y_t;\hat{X}_k|Q_t=q_t, P_t=p_t) 
    &\overset{(i)}{\leq} \ex{x_k, x_k' \leftarrow X_k}{ \KL\Big(\tilde{\cO}_{\bot}(G_{t,-k}(x_k)) || \cM(G_{t,-k}(x_k'))) \Big)} \nonumber \\
    &\overset{(ii)}{\leq}  \cnt_k^2(q_t) \rho. \nonumber
\end{align}
Inequality $(i)$ holds irregardless of any privacy notion, and so the consideration is inequality $(ii)$.
We observe that we only ever need to apply this bound for $\cnt_k^2(q_t) \leq \sqrt{d/\rho}$. In the other regime, the proof upper bounds the information via entropy,
$I(\Y_t;\hat{X}_k|Q_t=q_t, P_t=p_t) = O(d).$ 
For $m \leq \omega$, $(\rho,\omega)$-tCDP implies $(\rho m^2, \omega/m)$-group tCDP for groups of size $m$. Since tCDP also bounds KL divergence, this is sufficient to obtain inequality $(ii)$, and the rest of the proof proceeds exactly the same as in the zCDP case. Similarly, in the case where $\bar{m} \leq \sqrt{d/\rho}$, we observe that we only have to use the group privay properties of tCDP for groups of size at most $\bar{m}$.

\paragraph{Upper bound.} The fact that our upper bound, Algorithm \ref{alg:NMSGD}, satisfied tCDP is already proved as an intermediate step in the proof of Lemma \ref{lem:nmsgd-aggregate-privacy}. Specifically, for any setting of $\rho$ such that $\log(n\sqrt{\rho/d}) \geq 3\rho(2+\log_2(1/\rho)$, the algorithm is $(\rho',\omega)$-tCDP with $\rho' = O(\frac{1}{\alpha^2} \frac{d}{n^2})$ and $\omega = 1/\rho$. We remark that while many other upper bounds in the literature are stated for approximate DP, they often satisfy tCDP guarantees as they rely on zCDP oracles and subsampling.

\section{Supplement to Section \ref{sec:smooth}}
\subsection{Proof of Theorem \ref{thm:smooth-lb} (oracle complexity lower bound for smooth losses)} \label{app:smooth-lb}
Theorem \ref{thm:smooth-lb} follows from two runtime lower bounds we prove in this section. The first is an $\Omega\br{\frac{\rad\lip\sqrt{d}}{\alpha}}$ lower bound for private optimizers, and a $\Omega\br{\min\bc{\frac{\rad^2\lip^2}{\alpha^2}, n}}$ lower bound which holds even without privacy. Much of the private lower bound proof depends on a DP mean estimation lower bound for procedures which only access a limited number of samples from the dataset. These results are provided subsequently in Appendix \ref{app:dp-mean-lb}. 

\begin{theorem} \label{thm:smooth-dp-restatement}
Let $\delta\leq \frac{1}{16 n d}$, $\epsilon \leq \log(1/\delta)$, and $d$ be larger than some constant. Let $\cA$ be an $\alpha$-accurate optimizer for $(\cF^n_{\lip,\frac{\alpha}{\rad^2}},\cK_\rad)$ which is  $(\epsilon,\delta)$-DP. Further assume that, in expectation, $\cA$, makes at most $\runtime$ calls to the gradient oracle. Then, 
$\runtime = \Omega\bigro{\frac{\rad\lip\sqrt{d}}{\alpha}}.$
\end{theorem}
\begin{proof}
Define the distribution $\cD_\theta$, which for any vector $\theta \in [-1,1]^d$, is the product distribution where, for any $j\in[d]$, a sample has its $j$'th coordinate as $1$ with probability $(1+\theta_j)/2$ and as $-1$ with probability $(1-\theta_j)/2$.
Let $\Theta\sim \Unif([-1,1]^d)$. 
For $N=n\alpha/(BL)$, let $\tilde{x}_1,...,\tilde{x}_N \sim \frac{\lip}{\sqrt{d}}\cD_\Theta^N$. Now we take $x_1,...,x_n$ to be a random permutation of $\tilde{x}_1,...,\tilde{x}_N$ and $n-N$ zero vectors.
Define the loss function,
\begin{align}\label{eq:unc_lb_loss}
    \cL(w) = \frac{1}{n}\sum_{i=1}^n \ip{w}{x_i} + \lambda \|w\|^2.
\end{align}
We will use the constraint set $\cB(72\rad)$, and define $w^* = \argmin\limits_{w\in\cB(24\rad)}{\cL(w)}$.
Note that the \textit{unconstrained} empirical minimizer is $\tilde{w}^*= \frac{\sum_{i=1}^n x_i}{2n\lambda}$. 
Since $\norm{\sum_{i=1}^n x_i} \leq NL$, we set $\lambda \geq \frac{NL}{144 nB}$ so that $\norm{\tilde{w}^*} \leq 72B$ and thus $w^*=\tilde{w}^* = \frac{\sum_{i=1}^n x_i}{2n\lambda}$. 
Further, under the setting of $N$, we have $\lambda = \frac{\alpha}{144\rad^2} \leq \frac{\lip}{\rad}$, which ensures the loss is $2\lip$-Lipschitz over the set $\cB(72\rad)$.

Now we will show that any $w$ which achieves small excess risk is close to $w^*$. 
For any $w$ we have,
\begin{align*}
    \cL(w) - \cL(w^*) &= \Big\langle{w-w^*},{\frac{1}{n}\sum_{i=1}^n x_i}\Big\rangle + \lambda\br{\norm{w}^2 - \norm{w^*}^2} \\
    &= 2\lambda\ip{w-w^*}{-w^*} +  \lambda\br{\norm{w}^2 - \norm{w^*}^2} \\
    &= 2\lambda(\|w^*\|^2 -\ip{w}{w^*}) +  \lambda\br{\norm{w}^2 - \norm{w^*}^2} \\
    &= 2\lambda\br{\|w^*\|^2 -\frac{1}{2}\|w^*\|^2 - \frac{1}{2}\norm{w}^2 + \frac{1}{2}\norm{w-w^*}^2} +  \lambda\br{\norm{w}^2 - \norm{w^*}^2} \\
    &= \lambda\norm{w-w^*}^2.
\end{align*}
where the fourth equality comes from $\ip{a}{b} = \frac{1}{2}(\norm{a}^2 + \norm{b}^2 - \norm{a-b}^2)$.

Now observe $\frac{1}{\lip N}\sum_{i=1}^N x_i = \frac{2n\lambda}{\lip N} w^*$ and consider the mean estimation candidate $\bar{\Theta} = \frac{2n\lambda}{\lip N} w$, where $w$ is the output of the differentially private solver $\cA$. Continuing from the above display we then have,
\begin{align} \label{eq:loss-to-mean}
    \ex{\Theta,S,\cA}{\cL(w) - \cL(w^*)} &= \ex{}{\lambda\norm{w-w^*}^2} \nonumber \\
    &= \frac{\lip^2N^2}{n^2\lambda} \ex{}{\|\bar{\Theta}-\Theta_S\|^2} \nonumber \\
    &\geq \frac{\lip^2N^2}{4n^2\lambda} \ex{}{\|\bar{\Theta}-\Theta_S\|}^2.  
\end{align}
Now because the nonzero vectors are randomly assigned indices in $[n]$, and there are at most $N$ of them, the expected number of nonzero vectors in $S$ accessed by $\cA$ is $\frac{Ns}{n} = \frac{s\alpha}{\rad\lip}$. Assume by contradiction that $s < \frac{\rad\lip\sqrt{d}}{18\alpha\sqrt{\log(1/\delta)}}$, by Lemma \ref{lem:random-runtime-dp-mean-lb} (given in the following section) this means $\ex{}{\|\bar{\Theta}-\Theta_S\|}^2 \geq \frac{1}{36}$. 

Recalling $N=n\alpha/[\rad\lip]$ and $\lambda = \frac{N\lip}{144 n\rad}$, under the assumption that $s < \frac{\rad\lip\sqrt{d}}{18\alpha\sqrt{\log(1/\delta)}}$ and $\ex{}{\cL(w) - \cL(w^*)} \leq \alpha$, we have
\begin{align*}
    \alpha \geq \frac{\lip^2 N^2}{72n^2\lambda} = \frac{2 \rad\lip N}{n} = 2\alpha.
\end{align*}
This is a contradiction and so it must be that $\runtime \geq \frac{\rad\lip\sqrt{d}}{18\alpha\sqrt{\log(1/\delta)}}$.
\end{proof}

The non-private component of the lower bound comes from the following result. We note a similar result was proved in \cite{WS16}, and we provide the following only to extend it to algorithms with randomized running time.
\begin{lemma}\label{lem:nonDP-smooth-lb}
 Let $\cA$ be an $\alpha$-accurate optimizer for $(\cF^n_{\lip,\frac{\alpha}{\rad^2}},\cK_\rad)$. Further assume that, in expectation, $\cA$, makes at most $\runtime$ calls to the gradient oracle. Then, 
$\runtime = \Omega\bigro{\min\bigc{\frac{\rad^2\lip^2}{\alpha^2}, n}}.$
\end{lemma}
\begin{proof}
We will first give a \textit{distributional} mean estimation bound for estimators which use a variable number of samples from the distribution.
Let $T$ denote the number of samples used by the mean estimation procedure (which may be data dependent). Let $\cA_t$ denote the set of all algorithms which w.p. $1$ use at most $t$ samples. We have for any estimator $\cM$,
\begin{align*}
    \ex{}{\|\cM(S)-\Theta\|} &=  \sum_{\substack{\theta\in\Supp(\Theta)\\s\in\Supp(S)}}\sum_{t=1}^\infty \ex{\cM}{ \|\cM(S)-\Theta\| ~\big|~ T=t,\Theta=\theta,S=s} \PP[T=t | \Theta,S] \PP[\Theta=\theta,S=s] \\
    &\geq  \sum_{\theta,s}\ex{\cM}{ \|\cM(S)-\Theta\| ~\big|~ T\leq 10\runtime,\Theta=\theta,S=s} \PP[T\leq 2\runtime | \Theta=\theta,S=s] \PP[\Theta=\theta,S=s] \\
    &\geq \frac{1}{2}\sum_{\theta,s}\ex{\cM}{ \|\cM(S)-\Theta\| ~\big|~ T\leq 2\runtime,\Theta=\theta,S=s} \PP[\Theta=\theta,S=s] \\
    &= \frac{1}{2} \ex{\cM,\Theta,S}{\|\cM(S)-\Theta\| ~\big|~T \leq 2s} \\
    &\geq \frac{1}{2}\min_{\cM\in\cA_{2s}}\bc{\ex{\cM,\Theta,S}{\|\cM(S)-\Theta \|} }.
\end{align*}
The second inequality uses the bound on expected running time and Markov's inequality. Now, the fact that $\min\limits_{\cM\in\cA_{2s}}\{\ex{\cM,\Theta,S}{\|\cM(S)-\Theta \|}\} = \Omega(\frac{\lip}{\sqrt{s}})$ follows from classic mean estimation lower bounds. For example, using $\Theta \sim \Unif([-2/\sqrt{s}, 2/\sqrt{s}]^d)$ and $S\sim \frac{\lip}{\sqrt{d}}\cD_\Theta^n$ suffices by \cite[Theorem 13]{ullah2024publicdata}. Thus we have
\begin{align}\label{eq:minimax-mean}
    \ex{}{\|\cM(S)-\Theta\|} = \Omega\br{\frac{\lip}{\sqrt{s}}}.
\end{align}
We can now leverage this lower bound and the loss construction from Theorem \ref{thm:smooth-dp-restatement}. Letting $\bar{\Theta} = \frac{2n\lambda}{\lip N} w$ and $\Theta_S=\frac{1}{n}\sum_{x\in S}{x}$ and using the loss construction from Theorem \ref{thm:smooth-dp-restatement}/Eqn. \eqref{eq:unc_lb_loss} with $N=n$ and $\lambda = \alpha/\rad^2$, we obtain from Eqn. \eqref{eq:loss-to-mean} that,
\begin{align*}
    \alpha \geq \frac{\rad^2}{4\alpha}\ex{}{\|\bar{\Theta}-\Theta_S\|}^2 \implies \alpha \geq \rad\ex{}{\|\bar{\Theta}-\Theta_S\|}. 
\end{align*}
Clearly for any $\Theta$, $\ex{}{\|\Theta_S-\Theta\|} \leq \frac{\lip}{\sqrt{n}}$. Thus by Eqn. \eqref{eq:minimax-mean}, for some constant $C$,
\begin{equation*}
    \alpha \geq \rad\ex{}{\|\bar{\Theta}-\Theta_S\|} \geq \rad\br{\frac{C\lip}{\sqrt{s}} - \frac{\lip}{\sqrt{n}}} \implies s = \Omega\br{\min\bc{\frac{\rad^2\lip^2}{\alpha^2}, n}}. \qedhere
\end{equation*}
\end{proof}

\subsubsection{DP Mean estimation with variable access}\label{app:dp-mean-lb}
In this section, we provide lower bounds for DP mean estimation when the algorithm only accesses a random subset of the dataset. In the following, we will denote the distribution $\cD_\theta$, which for any vector $\theta \in [-1,1]^d$, is the product distribution where, for any $j\in[d]$, a sample has its $j$'th coordinate as $1$ with probability $(1+\theta_j)/2$ and as $-1$ with probability $(1-\theta_j)/2$.
\begin{lemma}
\label{lem:random-runtime-dp-mean-lb}
Let $\delta\leq \frac{1}{12 n d}$ and $\epsilon \leq \log(1/\delta)$. Let $\cA$ be an $(\epsilon,\delta)$-DP algorithm such that for any dataset $S\in \cB(1)^n$, in expectation $\cA$ accesses at most $\runtime$ elements of $S$. 
Draw 
$\Theta\sim\mathsf{Unif}([-1,1]^d)$
and 
$S = \bc{X_1,\ldots,X_n} \sim\frac{1}{\sqrt{d}}\cD_\Theta^n$.
It holds that,
    \begin{align*}
   \ex{\Theta,\cA,S}{\Big\|\cA(\spriv)-\frac{1}{n}\sum_{i=1}^n X_i\Big\|} \geq \frac{1}{6} 
   ~~~\text{or}~~~ \runtime \geq \frac{\sqrt{d}}{18\sqrt{\log(1/\delta)}}.
\end{align*}
\end{lemma}

To prove the lemma, we will use a standard result in the privacy lower bound literature, often called the fingerprinting lemma. This result stems from \cite{DSSUV15} and can be obtained more directly from \cite[Lemma 4]{ullah2024publicdata}. Our accuracy assumption differs slightly from theirs. This modification can be obtained by simply avoiding an application of Jenson's inequality at the end of their proof, which we have copied below for completeness.
\begin{lemma}\label{lem:fp-lemma} 
Let $\theta$ be sampled uniformly from $\Unif([-1,1]^d)$. Let $\cA$ satisfy $\|\ex{S\sim\cD^n_\theta}{\cA(S)}-\theta\|\leq \sqrt{d}/6$ for any $\theta\in[-1,1]^d$. Then %
one has,
\begin{align*}
    \ex{\cA,S,\Theta}{\sum_{i=1}^n \ip{\cA(S)}{X_i-\Theta}} &\geq \frac{d}{3}. %
\end{align*}
\end{lemma}
\begin{proof}
In the following we treat $\cA$ as a deterministic function and bound $\ex{S,\Theta}{\sum_{i=1}^n \ip{\cA(S)}{X_i-\Theta}}$. This is sufficient to bound $\ex{\cA,S,\Theta}{\sum_{i=1}^n \ip{\cA(S)}{X_i-\Theta}}$ for randomized $\cA$, since the analysis holds for any function (i.e. the distribution does not depend on $\cA$).
Further, we start with the one dimensional case such that $\Theta\in\mathbb{R}$.
Define $g(\Theta) = \ex{S\sim\cD^n_\Theta}{\cA(S)}$. We start by applying results developed in \cite{DSSUV15}, 
\begin{align*}
\ex{S,\Theta}{\cA(S)\sum_{i=1}^n(X_i-\Theta)} %
&\overset{(i)}{=} \ex{\Theta}{g'(\Theta)(1-\Theta^2)} \\
&\overset{(ii)}{\geq} 1 - \ex{}{\Theta^2} + 2\ex{\Theta}{(g(\Theta)-\Theta)\Theta} - \frac{|g(-1) + 1| + |g(1) - 1|}{2}\\
&\geq 2/3 + 2\ex{\Theta}{(g(\Theta)-\Theta)\Theta} - \frac{|g(-1) + 1| + |g(1) - 1|}{2}.
\end{align*} \\
Above, $(i)$ comes from \cite[Lemma 5]{DSSUV15} and $(ii)$ comes from \cite[Lemma 14]{DSSUV15}.
We now have
\begin{align*}
    \ex{S,\Theta}{\cA(S)\sum_{i=1}^n(X_i-\Theta)} &\geq 2/3 + \frac{|g(-1) + 1| + |g(1) - 1|}{2} + 2\ex{\Theta}{(g(\Theta)-\Theta)\Theta} \\
    &\geq 2/3 + \frac{|g(-1) - 1| + |g(1) - 1|}{2} - 2\ex{\Theta}{|g(\Theta)-\Theta|\cdot|\Theta|} \\
    &\geq 2/3 - \frac{|\mathbb{E}_{S\sim\cD_{-1}}[\cA(S)] + 1| + |\mathbb{E}_{S\sim\cD_{1}}[\cA(S)] - 1|}{2} \\
    &~~~~ - 2\ex{\Theta}{\Big| \ex{S\sim\cD_{\Theta}}{\cA(S)}-\Theta \Big|}.
\end{align*}
Above we use the fact that $|\Theta|\leq 1$ and the definition of $g$.

We can now extend the above analysis to higher dimensions. For $\Theta \in \mathbb{R}^d$, the above holds for each $\Theta_j$, $j\in[d]$. For convenience define  $\bar{1}=(1,\ldots,1)\in\mathbb{R}^d$. Summing over $d$ dimensions we have 
\begin{align*}
    &  \E_{S,\Theta}\Bigs{\Big\langle\cA(S),\sum_{i=1}^n(X_i-\Theta)\Big\rangle} \\
    &\geq \frac{2d}{3} - \frac{1}{2}\Big\|\ex{S\sim\cD_{-\bar 1}}{\cA(S)}+\bar{1}\Big\|_1 - \frac{1}{2}\Big\|\ex{S\sim\cD_{\bar 1}}{\cA(S)}-\bar{1}\Big\|_1 -  2\ex{\Theta}{\Big\|\ex{S\sim\cD_\Theta}{\cA(S)}-\Theta\Big\|_1} \\
    &\geq \frac{2d}{3} - \frac{1}{2}\|\ex{S\sim\cD_{-\bar 1}}{\cA(S)}+\bar{1}\|_1 - \frac{1}{2}\|\ex{S\sim\cD_{\bar 1}}{\cA(S)}-\bar{1}\|_1 -  2 \ex{\Theta}{\Big\|\ex{S\sim\cD_\Theta}{\cA(S)}-\Theta\Big\|_1} \\
    &\geq \frac{2d}{3} - \frac{\sqrt{d}}{2}\|\ex{S\sim\cD_{-\bar 1}}{\cA(S)}+\bar{1}\|_2 - \frac{\sqrt{d}}{2}\|\ex{S\sim\cD_{\bar 1}}{\cA(S)}-\bar{1}\|_2 -  2\sqrt{d}\ex{\Theta}{\Big\|\ex{S\sim\cD_\Theta}{\cA(S)}-\Theta\Big\|_2} \\
    &\geq \frac{d}{6}. %
\end{align*}
This proves the claim.   
\end{proof}

We can now prove the mean estimation lower bound. This proof follows a similar structure to existing proofs for DP mean estimation, although additional work must be done to account for the fact that $\cA$ only accesses a subset of points in the dataset.
\begin{proof}[Proof of Lemma \ref{lem:random-runtime-dp-mean-lb}] 
    For our proof we will use a dataset of vectors in $\bc{\pm 1}^d$, and as such the $\ell_2$ bound on the data is $\sqrt{d}$. The final result will follow from rescaling by $\frac{1}{\sqrt{d}}$.
    Condition on $\Theta=\theta$ and define the following random variables for each $i\in[n]$,
    \begin{align*}
        Z_i = \ip{\cA(\spriv)}{X_i-\theta} ~~~~~~\text{ and }~~~~~~ 
        Z_i'= \ip{\cA(S_{\sim i})}{X_i-\theta},
    \end{align*}
    where $S_{\sim i}$ is the dataset formed by replacing $i$-th data point of $\spriv$ with $X_i' \sim \cD_\theta$.

Let $I$ denote the random variable corresponding to the subset of indices of data points accessed by $\cA$. We have for some $\tau \geq 0$,
\begin{align*}
     \pr{}{Z_i \geq \tau | i\in I}\pr{}{i\in I} &= \pr{}{Z_i \geq \tau} - \pr{}{Z_i \geq \tau | i\notin I}\pr{}{i\notin I} \\
    &\leq e^\epsilon\pr{}{Z_i' \geq \tau} + \delta \\
    &\leq \exp\bigro{\epsilon-\frac{\tau^2}{8d}}+\delta. 
\end{align*}
The last inequality uses the Chernoff-Hoeffding bound.
Since $\epsilon\leq \log(1/\delta)$, for $\tau=\sqrt{3d\log(1/\delta)}$ we obtain,
$\pr{}{Z_i \geq \tau | i\in I} \leq \frac{2\delta}{\pr{}{i\in I}}.$
Using this we can derive,
\begin{align*}
    \ex{\cA,S}{Z_i} &= \ex{}{Z_j | i\in I}\pr{}{i \in I} + \ex{}{Z_j | i\notin I}\pr{}{i\notin I} \\
    &= \ex{}{Z_j | i\in I}\pr{}{i \in I} \\
    &\leq \pr{}{i \in I}\br{\sqrt{3d\log(1/\delta)} + 2d\pr{}{Z_j > \tau | i \in I}} \\
    &\leq \pr{}{i \in I}\sqrt{3d\log(1/\delta)} + 4d\delta.
\end{align*}
Above we use the fact that the expectation of $Z_i$ is $0$ when $\cA$ does not access the $i$'th element. Now for the sum we have,
\begin{align*}
    \ex{\cA,S}{\sum_{i=1}^n Z_i} &\leq 4de^\epsilon\delta + \sqrt{3d\log(1/\delta)}\sum_{i=1}^n \pr{}{i\in I} \\
    &= 4nd\delta + \sqrt{3d\log(1/\delta)}\ex{}{\sum_{i=1}^n \ind{i \in I} } \\
    &= 4nd\delta + \sqrt{3d\log(1/\delta)}\ex{}{|I|} \\
    &= 4nd\delta + \runtime\sqrt{3d\log(1/\delta)} \\
    &\leq 3s\sqrt{d\log(1/\delta)}.
\end{align*}
We then obtain the same upper bound for $\ex{\Theta,\cA,S}{\sum_{i=1}^n Z_i}$.
We now use the fingerprinting lemma to lower bound the correlation. Specifically, in the case where $\cA$ is at least $\sqrt{d}/6$ accurate (which we note corresponds to $1/6$ accurate after rescaling), we have for any $\theta$,
$\|\ex{}{\cA(S)}-\theta\| = \|\ex{}{\cA(S)-\frac{1}{n}\sum_{i=1}^n X_i}\| \leq \ex{}{\|\cA(S)-\frac{1}{n}\sum_{i=1}^n X_i\|} \leq \frac{\sqrt{d}}{6}$. Thus by Lemma \ref{lem:fp-lemma},
\begin{align*}
    \mathbb{E}\Big[\sum_{i=1}^nZ_i \Big] = \E\Big[\Big\langle\cA(\spriv),\sum_{i=1}^n X_i-\Theta\Big\rangle\Big] \geq \frac{d}{6}.
\end{align*}
Now using the upper and lower bounds on $\mathbb{E}\Big[\sum_{i=1}^nZ_i \Big]$ we obtain,
\begin{align*}
    3\runtime\sqrt{d\log(1/\delta)} &\geq d/6 \implies \runtime \geq \frac{\sqrt{d}}{18\sqrt{\log(1/\delta)}}. \qedhere
\end{align*}
\end{proof}

\subsection{Proof of Theorem \ref{thm:smooth-ub} (upper bounds for smooth losses)} \label{app:smooth-ub}

\begin{algorithm}[H]
\caption{\label{alg:phased-sgd}Phased SGD}
\begin{algorithmic}[1]
\REQUIRE Accuracy $\alpha\geq 0$, Oracle $\cO$ for losses $\ell_1,...,\ell_n$, 
Constraint Set $\cW$ of width $B$, Lipschitz constant $L$, Privacy parameter $\delta \in [0,1]$
\STATE Pick any $w_{0}\in\cW$

\STATE Set $R= \frac{1}{2}\log_2(1/\alpha)$

\STATE Set $T = \max\bc{\frac{\rad\lip\sqrt{d\log(n/\delta)}
}{\alpha}, \frac{\rad^2\lip^2}{\alpha^2}}$  and $\eta = \frac{\rad}{\lip}\min\bc{\frac{1}{ \sqrt{d\log(n/\delta)}}, \frac{\alpha}{\rad\lip}}$ 

\FOR{$r=1...R$}

\STATE Set $T_r = 2^{-r}T$ and $\eta_r = 4^{-r}\eta$
\STATE Run SGD over $\cW$ from $w_{r-1}$ for $T_r$ steps with learning rate $\eta_i$. Let $\bar{w}_r$ be the average iterate

\STATE $w_r = \bar{w}_r + \xi_r$, with $\xi_r \sim \cN(0,\mathbb{I}_d\sigma^2_r)$ and $\sigma_r = \frac{4\rad}{4^r\sqrt{d}}$ %

\ENDFOR
\STATE \textbf{Output:} $w_R$
\end{algorithmic}
\end{algorithm}

The SGD algorithm used as a subroutine in Algorithm \ref{alg:phased-sgd} starts at some point $w_0\in\cW$, and at each step samples $i\sim\Unif([n])$ and performs the update $w_t = \Pi_{\cW}(w_{t-1} - \eta \nabla \ell_i(w))$.
\begin{theorem} \label{thm:smooth-ub-restatement}
For any $\alpha,\delta > 0$, $\cA$ is $O(\alpha)$-accurate for $(\cF^n_{\lip,\frac{\lip^2}{\alpha}},\cK_\rad)$. %
Further, the algorithm
uses at most $\max\big\{\frac{\rad\lip\sqrt{d\log(n/\delta)}
}{\alpha}, \frac{\rad^2\lip^2}{\alpha^2}\big\}$ oracle evaluations. For $\epsilon\in[0,1]$, if $\alpha \geq 6\alpha_{\epsilon,\delta}^*$ it satisfies $(\epsilon,\delta)$-DP. 
\end{theorem}
\begin{proof}
For notation, define $\bar{w}_0=w^*$ and $\xi_0=w_0-w^*$.
Using the convergence results of SGD. The error can be decomposed via,
\begin{align*}
    \ex{}{\cL(w_R) - \cL(w^*)} &= \ex{}{\cL(w_R)-\cL(\bar{w}_R)} + \sum_{i=1}^R \ex{}{\cL(\bar{w}_r)-\cL(\bar{w}_{r-1})} \\
    &\leq \lip\ex{}{\|\xi_R\|} +  \sum_{r=1}^R \frac{\ex{}{\|\xi_{r-1}\|^2}}{2\eta_i T_i} + \frac{\eta_i\lip^2}{2} \\
    &\leq \frac{\rad\lip}{2^{-2R}} + \sum_{r=1}^R 2^{-r}\br{\frac{8\rad^2}{\eta T} + \frac{\eta \lip^2}{2}} \\
    &\leq \alpha + \frac{8\rad^2}{\eta T} + \frac{\eta \lip^2}{2} \\
    &= O(\alpha).
\end{align*}

The first inequality comes from standard convergence guarantees of projected SGD, see e.g. \cite[Theorem 14.8]{shalev2014understanding}.

For the privacy analysis, we will leverage privacy amplification via subsampling (without replacement) results. Specifically we will use an extended version of \cite[Lemma 4.14]{bun-dp-thresholds}, restated as Lemma \ref{lem:amp-wor} in Appendix \ref{app:extra-lemmas}. 

Consider the mechanism which, upon receiving $m>0$ losses, $\ell_1,...,\ell_m$, performs one-pass SGD over the losses, then adds isotropic Gaussian noise with variance $\sigma_r^2$. 
Assuming each $\ell_i$,~$i\in[n]$, is at least $1/(2\eta)$-smooth, $\bar{w}_r$ has sensitivity (w.r.t. changing one of $\{\ell_1,\ldots,\ell_m\}$) at most $2\lip\eta_r$, see e.g. \cite[Lemma 3.6]{HRS15}. 
As such, the Gaussian mechanism is $(\epsilon_r,\delta/n)$-DP with respect to a change in one sampled loss function, with 
$\epsilon_r = \frac{4\lip\eta_r\sqrt{\log(n/\delta)}}{\sigma_r}$.

Now observe that Algorithm \ref{alg:phased-sgd}, at each phase, applies the previously described mechanism to a batch of $T_i$ losses, sampled with replacement from $\{\ell_1,\ldots,\ell_n\}$. 
In the regime $\alpha \geq \frac{\rad\lip}{\sqrt{d\log(n/\delta)}}$, we have $T=\frac{\rad\lip\sqrt{d\log(n/\delta)}
}{\alpha}$ and $\eta = \frac{\rad}{\lip\sqrt{d\log(n/\delta)}}$, and thus $\forall r\in[R], T_r \leq n/2$ and $\epsilon_r \leq \frac{1}{2^r}$. Alternatively, in the other regime we have $T = \frac{\rad^2\lip^2}{\alpha^2}$ and $\eta = \frac{\alpha}{\lip^2}$, and thus $\epsilon_r \leq \frac{1}{2^r}\frac{\alpha\sqrt{d}}{\rad\lip}.$ Observe $\frac{1}{2^r}\frac{\alpha\sqrt{d}}{\rad\lip} \leq \frac{n}{2 T_i}$ for any $\alpha \geq \alpha_{1,\delta}^*$. In either case, $\epsilon_r \leq \frac{n}{2T_i}$, and thus we can apply the amplification via subsampling result from Lemma \ref{lem:amp-wor}.
Specifically, 
this implies that each round of the algorithm is $(\epsilon_r',\delta_r')$ with,
\begin{align*}
 \epsilon_r' &= \frac{6T_i}{n}\epsilon_r \leq \frac{6}{R}\max\bc{{\frac{\rad\lip\sqrt{d\log(n/\delta)}}{n\alpha}, \frac{\rad^2\lip^2}{\alpha^2 n} \cdot \frac{\alpha \sqrt{d\log(n/\delta)}}{\rad\lip}}} = 6 \frac{\alpha^*_\delta}{2^r \alpha}, \\
 \delta_r' &= e^{6\epsilon' T_i/n}\frac{4T_i}{n} \frac{\delta}{n} \leq e^62^{-r}\delta.
\end{align*}
By composition, the overall privacy of the algorithm satisfies $(\epsilon,\delta)$-DP with
$\epsilon \leq 6\sum_{r=1}^R 2^{-r} \frac{\alpha^*_{1,\delta}}{\alpha} \leq 6\frac{\alpha^*_{1,\delta}}{\alpha}$.
\end{proof}

\begin{remark}
The Phased SGD algorithm also shows why one must assume the queries sent to the private proxy oracle are non-adaptive for our lower bound in the non-smooth case to be hold. An inspection of the privacy analysis in the proof of Theorem \ref{thm:smooth-ub} shows that, if one only cares that the algorithm is private with respect to its dataset of gradients, smoothness is not needed. We emphasize that being private with respect to the dataset of gradients does not imply a general DP optimizer however, and indeed Phased SGD is not known to be DP is non-smooth case.
\end{remark}

\subsubsection{$\tilde{O}(n)$ running time algorithm when $\alpha = O\bigro{\frac{\rad\lip}{\sqrt{n}}}$.} \label{app:linear-smooth-ub}
To achieve near linear running time, one case use the Phased-ERM algorithm of \cite{FKT20} (Algorithm 3 therein) in conjunction with accelerated ERM solvers. This was essentially shown by \cite{ZTOH22}, although because they studied DP-SCO, they only explicitly stated results for error $\alpha \geq \frac{\rad\lip}{\sqrt{n}}$. Nonetheless, their technique translates just as well for smaller error when considering DP-ERM. We describe this in the following.

Given some (non-private) ERM solver, $\cA$, which solves a strongly convex ERM problem to high accuracy, the Phased-ERM algorithm, Algorithm \ref{alg:phased-erm}, is differentially private and yields an accurate solution. Precisely, we have the following result.
\begin{lemma}\label{lem:phased-erm}
Let $\delta \in [0,\rad\alpha^2]$ and $\epsilon\in[0,1]$. Algorithm \ref{alg:phased-erm} is $O(\alpha)$-accurate for $(\cF^n_{\lip,\infty},\cK_\rad)$ and for $\alpha \geq \log(n)\alpha_{\epsilon,\delta}^*$ it satisfies $(\epsilon,2\log(n)\delta)$-DP. 
\end{lemma}
The proof follows similarly to the one in \cite{FKT20}, and is given below.
When the loss is additionally $\beta$-smooth, there are solvers for the regularized subproblem (such as SVRG, \cite{SVRG}) which achieve the accuracy condition in $O((n+\rad^2\beta/\alpha)\log(n/\alpha))$ oracle calls. For $\alpha \geq \alpha^*_{1,\delta}$, we get near linear running time if 
$\beta \leq \lip\sqrt{d}/\rad\}$.

\begin{algorithm}[H]
\caption{\label{alg:phased-erm}Phased ERM}
\begin{algorithmic}[1]
\REQUIRE Accuracy $\alpha\geq 0$, Oracle $\cO$ for losses $\ell_1,...,\ell_n$, 
Constraint Set $\cW$ of width $B$, Lipschitz constant $L$, Parameter $\delta \in [0,1]$
\STATE Pick any $w_{0}\in\cW$

\STATE Set $R= \log_2(\lip\rad/\alpha)$

\STATE Set $\lambda_r = 2^r\frac{\alpha}{\rad^2}$ for all $r\in[R]$

\FOR{$r=1...R$}

\STATE Define $\cL_r(w) = \frac{1}{n}\sum_{i=1}^n \ell_i(w) + \lambda_r\|w-w_{r-1}\|^2$ and $w^*_r = \argmin\limits_{w\in\cW}\{\cL_r(w)\}$

\STATE Compute $\bar{w}_r$ such that w.p. at least $1-\delta$, $\cL_r(w)-\cL_r(w^*_r) \leq \min\bigc{\frac{\lip^2}{\lambda_r n^2}, 2^{-r}\alpha }$

\STATE $w_r = \bar{w}_r + \xi_r$ where $\xi_r \sim \cN(0,\mathbb{I}_d\sigma_r^2)$ and $\sigma_r = \frac{4\rad}{2^r\sqrt{d}}$

\ENDFOR
\STATE \textbf{Output:} $w_R$
\end{algorithmic}
\end{algorithm}
\begin{proof}[Proof of Lemma \ref{lem:phased-erm}]
For notation, let $\bar{w}_0 = w^*$. We have,
\begin{align*}
    \ex{}{\cL(w_R) - \cL(w^*)} &= \ex{}{\cL(w_R)-\cL(\bar{w}_R)} + \sum_{i=1}^R \ex{}{\cL(\bar{w}_r)-\cL(\bar{w}_{r-1})} \\
    &\leq 4\alpha + \sum_{i=1}^R \ex{}{\cL_r(\bar{w}_r)-\cL_r(\bar{w}_{r-1})}.
\end{align*}
The inequality uses $\sigma_R = \frac{4\alpha}{\lip\sqrt{d}}$ to bound $\lip \ex{}{\|\xi_R\|}$. We have for any $r\in[R]$,
\begin{align*}
    \ex{}{\cL(\bar{w}_r)-\cL(\bar{w}_{r-1})} &= \ex{}{\cL_r(\bar{w}_r) - \cL_r(\bar{w}_{r-1}) + \cL(\bar{w}_r)-\cL_r(\bar{w}_r) + \cL_r(\bar{w}_{r-1}) - \cL(\bar{w}_{r-1})} \\
    &\leq (2^{-r}\alpha + \delta\lip) + \lambda_r\ex{}{\|\bar{w}_{r-1}-w_{r-1}\|^2} \\
    &\leq 2^{-r}\alpha+\delta\lip + 2^r\frac{\alpha}{\rad^2}\bigro{\frac{16\rad^2}{2^{2r}}}.
\end{align*}
Thus we have $\ex{}{\cL(\bar{w}_r)-\cL(\bar{w}_{r-1})} \leq 2^{-r}18\alpha$ provided $\delta \leq \rad\alpha^2$, and the accuracy guarantee follows by combining the both displays.

For the privacy analysis, consider some $r\in[R]$. We have by standard results on the stability of regularized ERM that each $w^*_r$ is $\frac{\lip}{\lambda n}$-stable (i.e. changing one loss in $\cL$ perturbs $w_r^*$ by at most $\frac{\lip}{\lambda n}$) \cite{bousquet2002stability}. The $\lambda_r$ strong-convexity of $\cL_r$ and accuracy condition also implies that, conditioning on the randomness in previous rounds, with probability at least $1-\delta$, $\|\bar{w} - w^*_r\| \leq \sqrt{\frac{\cL_r(w)-\cL_r(w^*_r)}{\lambda_r}} \leq \frac{\lip}{\lambda_r n}$. So each $\bar{w}_r$ is $\Delta$-stable, with $\Delta = \frac{2\lip}{\lambda_r n} = \frac{2\rad^2}{2^r\alpha n} \leq \frac{2\rad \epsilon}{2^r \sqrt{\log(n)d\log(1/\delta)}}$, where the inequality follows from $\alpha \geq \log(n)\alpha^*_{\epsilon,\delta}$. Thus the Gaussian mechanism ensures each round $r$ is $(\epsilon/\log(n),2\delta)$-DP and by composition the overall algorithm is $(\epsilon,2\log(n)\delta)$-DP.  
\end{proof}

\section{Between DP-SCO to DP-ERM}\label{app:between-sco-and-erm}
In this appendix, we show via a reduction that (DP)-ERM and (DP)-SCO are equally hard in terms of runtime, up to log factors. This implies similar reductions for non-private settings as well.
\subsection{Reducing DP-SCO to DP-ERM}\label{app:erm-to-sco}
The following shows that DP-SCO is no harder than DP-ERM (up to log factors) via a reduction. 
Specifically, given an algorithm, $\cA$, which solves DP-ERM for $\rad=\lip=1$, we show how to construct a DP-SCO algorithm which has similar running time, using only black box access to $\cA$.

\begin{theorem}\label{thm:erm-to-sco}
Let $n,\rad,\lip\geq 0$, $\beta \in \re^+\cup \{\infty\}$, $\beta' = \beta + \frac{\lip}{\rad}$, and $n'=\frac{n}{\log(n)}$. Let $\cA$ be an $(\epsilon,\delta)$-DP algorithm which is $\alpha$-accurate for $(\cF_{5\lip,\beta'}^{n'},\cK_\rad)$ for DP-ERM and has expected running time $T$. Then there exists an algorithm which, using black box access to $\cA$, is $(\tilde{O}(\epsilon),\tilde{O}(\delta))$-DP and $\tilde{O}(\alpha + \frac{\rad\lip}{\sqrt{n}} + \frac{\rad\lip}{n\epsilon})$-accurate for $(\cF_{\lip,\beta}^{n},\cK_\rad)$ for DP-SCO, and has expected running time $\tilde{O}(T + \frac{\rad^2\lip^2}{\alpha^2})$. 
\end{theorem}
Note that any general DP algorithm for minimizing the population risk of non-smooth losses must incur error at least $\frac{\rad\lip}{\sqrt{n}} + \frac{\rad\lip}{n\epsilon}$ and runtime at least $\frac{\rad^2\lip^2}{\alpha^2}$.
Thus these additive factors are no worse than what one would obtain with a ``direct'' algorithm for DP-SCO. In the smooth case, the running time lower bound for finite sums is $\min\{\frac{\rad^2\lip^2}{\alpha^2},n\}$, and since $\alpha \geq \frac{\rad\lip}{\sqrt{n}}$, we again see that their is no asymptotic loss in runtime.

We now show how to obtain the theorem using the following result from \cite{BGM23}. We borrow parameter definitions from the above theorem statement. %
As an aside, we note that statements similar to Theorem \ref{thm:erm-to-sco} in different geometries are likely obtainable using a generalization of this statement provided in \cite{BGM24}.
\begin{theorem}\label{thm:bgm-reduction} \cite[Theorem 1]{BGM23}
Let $\cA$ be an algorithm which, given $D \in [\rad\sqrt{\frac{\log(n)}{n}},\rad]$ and randomly generated point $w'\in\cW$, satisfies $\ex{w',\cA}{\cL(\cA(\cO_\cL)) - \cL(w^*)} \leq \hat{\alpha}D$ whenever $\ex{w'}{\|w' - w^*\|} \leq D$ and any $\cL\in\cF_{5\lip,\beta'}^{n'}$.%
Then there exists an algorithm, which interacts with $\cL$ through at most $\log(n)$ calls to $\cA$ and is $\alpha$-accurate for $(\cF^n_{\lip,\beta},\cK_\rad)$ for SCO with
$\alpha=O\big(\log(n)\rad\hat{\alpha} + \frac{\log^{3/2}(n)\rad\lip}{\sqrt{n}}\big)$.
\end{theorem}
The original statement in \cite{BGM23} assumes the accuracy condition holds for all $D>0$, but an inspection of their proof shows that the relative accuracy condition is only used in their Eqn. 12 and for $D \in [\rad\sqrt{\frac{\log(n)}{n}},\rad]$. Further, \cite{BGM23} studied the more general case of saddle point problems, but ERM can be recovered by assuming range of the dual parameter is a singleton. Finally, we note that their algorithm only requires running the subroutine $\cA$ on regularized version of the loss, which, under their level of regularization, increases the smoothness parameter of the loss by at most $\frac{\lip}{\rad}$.

It has essentially already been shown in \cite[Section 5]{ABGMU22} how to obtain a DP algorithm satisfying the accuracy condition of Theorem \ref{thm:bgm-reduction} using black box access to a DP constrained optimizer, although their setting differs slightly. We provide a self contained version of their argument below. We will also make use of Fact \ref{fact:rescaling} in Appendix \ref{app:extra-lemmas} several times.

\begin{proof}[Proof of Theorem \ref{thm:erm-to-sco}]
In the following, let $T$ denote the expected running time of an $\alpha$-accurate DP-ERM algorithm, $\cA$, in the case where $\rad=\lip=1$.

We fist boost the expected empirical risk guarantee of $\cA$ into a high probability guarantee.
Let $u_0,\ldots,u_K$ be the result of $K=\log(n)$ independent runs of $\cA$ on $S$. By Markov's inequality, at least one of these runs achieves excess risk $2\alpha$ with probability at least $1-\frac{1}{2^K} = 1-\frac{1}{n}$. For each run $j\in[K]$, we generate a loss estimate, $E_j$, by sampling (without replacement) a minibatch of $1/\alpha^2$ losses from $\cL$ and computing the average loss on $u_j$. Since the range of the losses is $1$-bounded, we have by Chernoff-Hoeffding that,
$\prnos{}{E_j - \cL(u_j) \geq \sqrt{\frac{8\log(n)}{n}}} \leq \frac{1}{n^2}$. We then apply the exponential mechanism with privacy parameter $\epsilon$ over the scores $E_1,\ldots,E_K$ to select the solution candidate from $u_1,\ldots,u_K$. The guarantees of the exponential mechanism ensures that with probability at least $1-1/n$ the selected solution has loss within $\frac{4\log(n)}{n\epsilon}$ of the minimal loss candidate. Thus we obtain an accurate solution with probability at least $1-O(1/n)$ via a procedure that is via an algorithm that is $((\log(n)+1)\epsilon,\log(n)\delta)$-DP. The expected running time of this procedure is $\log(n)T + \frac{\log(n)}{\alpha^2}$.

Applying Fact \ref{fact:rescaling}, we can assume access to an algorithm $\tilde{\cA}$, which with probability at least $1-O(1/n)$ achieves accuracy $\rad\lip\tilde{\alpha}$ where $\tilde{\alpha}=\big(\alpha + \sqrt{\frac{8\log(n)}{n}} + \frac{4\log(n)}{n\epsilon}\big)$ on problems which are $\lip$-Lipschitz and have constraint set of radius at most $\rad$. 

We now describe how to use $\tilde{\cA}$ to obtain relative accuracy.
Letting $R=\frac{1}{2}\log(n)$, we run $\cA$ on $\cW_0,...,\cW_{R}$, to obtain candidate solutions $w_0,...,w_R$, where $\cW_r = \cW \cap \bc{w: \|w-w'\| \leq 2^{-r}\rad}$. Observe $w^*\in\cW_r$ for any $r \leq \frac{\log(1/\|w'-w^*\|)}{\log(2\rad)}$. 
We then pick the best candidate using the same loss estimate/exponential mechanism procedure used in the boosting argument. 
It is then easy to see that the solution selected by the exponential mechanism achieves excess empirical risk 
 that is $O\Big(\|w'-w^*\| \lip \tilde{\alpha} + \rad\lip\big(\sqrt{\frac{\log(n)}{n}} + \frac{\log(n)}{n\epsilon}\big)\Big)$.
Converting the high probability guarantee to expectation we obtain excess empirical risk
$O\Big(\|w'-w^*\| \lip \tilde{\alpha} + \rad\lip\big(\sqrt{\frac{\log(n)}{n}} + \frac{\log(n)}{n\epsilon} + \frac{1}{n}\big)\Big)$.
By taking expectation w.r.t. $w'$ we see the  condition of the theorem is satisfied with $\hat{\alpha}=O(\lip\tilde{\alpha})$. Theorem \ref{thm:erm-to-sco} then follows by applying Theorem \ref{thm:bgm-reduction} to the previously described algorithm.
\end{proof}

\subsection{Reducing DP-ERM ot DP-SCO} \label{app:sco-to-erm}
The reverse direction was given by \cite[Appendix C]{BFTT19}, and we here note that this direction can be performed without loss of log factors via a slightly different analysis. Specifically, for $\epsilon \leq 1/6$ and $\delta\in[0,1]$, given an $(\epsilon,\delta)$-DP-SCO solver $\cA_{SCO}$, we first sample $n$ points from the empirical distribution over $\{\ell_1,\ldots,\ell_n\}$, then runs the DP-SCO solver on the resampled dataset. By results for privacy amplification via subsampling with replacement (see Lemma \ref{lem:amp-wor} in Appendix \ref{app:extra-lemmas}), the result is $(6\epsilon,e\delta)$-DP. The bound on excess empirical risk follows directly from the accuracy guarantees of $\cA_{SCO}$, since it is run on i.i.d. samples from the empirical distribution.

\end{document}